\pdfoutput=1
\documentclass[letterpaper]{article}
\usepackage[totalwidth=480pt, totalheight=680pt]{geometry}

\usepackage{authblk}
\usepackage[utf8]{inputenc} % allow utf-8 input
\usepackage[T1]{fontenc}    % use 8-bit T1 fonts
\usepackage[dvipsnames]{xcolor}         % colors
\usepackage[pagebackref,colorlinks=true,urlcolor=blue,linkcolor=blue,citecolor=OliveGreen]{hyperref}             % hyperlinks
\usepackage{url}            % simple URL typesetting
\usepackage{booktabs}       % professional-quality tables
\usepackage{amsfonts}       % blackboard math symbols
\usepackage{nicefrac}       % compact symbols for 1/2, etc.
\usepackage{microtype}      % microtypography

\title{The Complexity of Sparse Tensor PCA\thanks{This project has received funding from the European Research Council (ERC) under the European Union’s Horizon 2020 research and innovation programme (grant agreement No 815464). This research/project is supported by the National Research Foundation, Singapore under its AI Singapore Programme (AISG Award No: AISG-PhD/2021-08-013).}}

\author[1]{Davin Choo\thanks{Part of the work was done while the author was in ETH Z\"urich.}}  
\author[2]{Tommaso d'Orsi}
\affil[1]{National University of Singapore, Singapore}
\affil[2]{ETH Z\"{u}rich, Switzerland}
\date{}

\usepackage{mathtools}
\usepackage{amsmath}
\usepackage{amssymb}
\usepackage{amsthm}
\usepackage{bbm}
\usepackage{enumerate}
\usepackage{algorithm}
\usepackage{algpseudocode}
\usepackage{mdframed}

\usepackage[capitalise,nameinlink]{cleveref}
\usepackage{thmtools}
\usepackage{thm-restate}
\newtheorem{theorem}{Theorem}

\newtheorem{proposition}[theorem]{Proposition}
\newtheorem{remark}[theorem]{Remark}
\newtheorem{corollary}[theorem]{Corollary}
\newtheorem{definition}[theorem]{Definition}
\newtheorem{conjecture}[theorem]{Conjecture}
\newtheorem{problem}[theorem]{Problem}
\newtheorem{claim}[theorem]{Claim}
\newtheorem{lem}[theorem]{Lemma}

\newtheorem{myremark}[theorem]{Remark}
\newtheorem*{remark*}{Remark}

\newcommand{\tensor}{\otimes}
\newcommand{\tensorpower}[2]{#1^{\tensor #2}}

\newcommand{\E}{\mathbb{E}}
\newcommand{\Var}{\mathbb{V}}
\newcommand{\N}{\mathbb{N}}
\newcommand{\R}{\mathbb{R}}
\newcommand{\bbP}{\mathbb{P}}
\newcommand{\cO}{\mathcal{O}}
\newcommand{\cI}{\mathcal{I}}
\newcommand{\cN}{\mathcal{N}}
\newcommand{\cS}{\mathcal{S}}
\newcommand{\cX}{\mathcal{X}}
\newcommand{\cY}{\mathcal{Y}}

\newcommand{\Paren}[1]{\left(#1\right)}
\newcommand{\Brac}[1]{\left[#1\right]}
\newcommand{\set}[1]{\{#1\}}
\newcommand{\Set}[1]{\left\{#1\right\}}
\newcommand{\Norm}[1]{\left\lVert#1\right\rVert}
\newcommand{\Abs}[1]{\left\lvert#1\right\rvert}
\newcommand{\card}[1]{\lvert#1\rvert}
\newcommand{\Card}[1]{\left\lvert#1\right\rvert}

\renewcommand{\Pr}{\bbP}
\newcommand{\eps}{\epsilon}

\newcommand{\iprod}[1]{\langle#1\rangle}
\newcommand{\Iprod}[1]{\left\langle#1\right\rangle}
\newcommand*{\dyad}[1]{#1#1{}^{\mkern-1.5mu\mathsf{T}}}

\DeclareMathOperator{\argmax}{argmax}
\DeclareMathOperator{\supp}{supp}

\providecommand{\nulld}{\nu}
\providecommand{\planted}{\mu}
\providecommand{\Ep}{\E_{\planted}}
\providecommand{\En}{\E_{\nulld}}

\providecommand{\hermitepoly}[2]{H_{#2}\Paren{#1}}
\providecommand{\lowdegpolys}[1]{\R[\pmb Y]_{\leq#1}}

\begin{document}

\maketitle

\begin{abstract}
We study the problem of sparse tensor principal component analysis: given a tensor $\pmb Y = \pmb W + \lambda \tensorpower{x}{p}$ with $\pmb W \in \otimes^p\R^n$ having i.i.d.\ Gaussian entries, the goal is to recover the $k$-sparse unit vector $x \in \R^n$.
The model captures both sparse PCA (in its Wigner form) and tensor PCA.

For the highly sparse regime of $k \leq \sqrt{n}$, we present a family of algorithms that smoothly interpolates between a simple polynomial-time algorithm and the exponential-time exhaustive search algorithm. 
For any $1 \leq t \leq k$, our algorithms recovers the sparse vector for signal-to-noise ratio $\lambda \geq \widetilde{\cO} (\sqrt{t} \cdot (k/t)^{p/2})$ in time $\widetilde{\cO}(n^{p+t})$, capturing  the state-of-the-art guarantees for the matrix settings (in both the polynomial-time and sub-exponential time regimes).

Our results naturally extend to the case of $r$ distinct $k$-sparse signals with disjoint supports, with guarantees that are independent of the number of spikes.
Even in the restricted case of sparse PCA, known algorithms only recover the sparse vectors for $\lambda \geq \widetilde{\cO}(k \cdot r)$ while our algorithms require $\lambda \geq \widetilde{\cO}(k)$.

Finally, by analyzing the low-degree likelihood ratio, we complement these algorithmic results with rigorous evidence illustrating the trade-offs between signal-to-noise ratio and running time.
This lower bound captures the known lower bounds for both sparse  PCA and tensor PCA.
In this general model, we observe a more intricate three-way trade-off between the number of samples $n$, the sparsity $k$, and the tensor power $p$.

\end{abstract}

\section{Introduction}
\label{sec:introduction}

Sparse tensor principal component analysis is a statistical primitive generalizing both sparse PCA\footnote{Often in the literature, the terms sparse PCA and spiked covariance model refer to the sparse spiked \emph{Wishart} model. However, here we  consider the sparse spiked \emph{Wigner} matrix model.} and tensor PCA\footnote{Tensor PCA is also known as the spiked \emph{Wigner} tensor model, or simply the spiked tensor model.}.
We are given multi-linear measurements in the form of a tensor
\begin{equation}\label{eq:generic-sparse-tensor-pca}\tag{SSTM}
	\pmb Y = \pmb W + \lambda \tensorpower{x}{p} \in \otimes^p\R^n
\end{equation}
for a Gaussian noise tensor $\pmb W \in \otimes^p\R^n$ containing i.i.d.\ $N(0,1)$ entries\footnote{Throughout the paper, we will write random variables in boldface.} and signal-to-noise ratio $\lambda > 0$.
Our goal is to estimate the ``structured'' unit vector $x \in \R^n$.
The structure we enforce on $x$ is sparsity: $\Card{\supp(x)}\leq k$.
The model can be extended to include multiple spikes in a natural way: $\pmb Y = \pmb W + \sum_{q=1}^r \lambda_q \tensorpower{x_{(q)}}{p}$, and even general order-$p$ tensors:
$\pmb Y = \pmb W + \sum_{q=1}^r \lambda_q \mathcal{X}_{(q)}$
for $\mathcal{X}_{(q)} = x_{(q,1)}\otimes\cdots\otimes x_{(q,p)}\in \otimes^p\R^n$.
In this introduction, we focus on the simplest single spike setting of \ref{eq:generic-sparse-tensor-pca}.

It is easy to see that sparse PCA corresponds to the setting with tensor order $p=2$.
On the other hand, tensor PCA is captured by effectively removing the sparsity constraint: $\Card{\supp(x)} \leq n$.
In recent years, two parallel lines of work focused respectively on sparse PCA \cite{johnstone2009consistency,amini2008high,berthet2013complexity,deshpande2016sparse,hopkins2017power,ding2019subexponential, DBLP:conf/colt/HoltzmanSV20, dOrsi2020} and tensor PCA \cite{richard2014statistical,hopkins2015tensor,ma2016polynomial,hopkins2017power,kunisky2019notes,arous2019landscape}, however no result captures both settings.
The appeal of the \emph{sparse spiked tensor model} (henceforth \ref{eq:generic-sparse-tensor-pca}) is that it allows one to study the computational and statistical aspects of these other fundamental statistical primitives in a unified framework, understanding the computational phenomena at play from a more general perspective.

In this work, we investigate  \ref{eq:generic-sparse-tensor-pca} from both algorithmic and computational hardness perspectives.
Our algorithm improves over known tensor algorithms whenever the signal vector is highly sparse.
We also present a lower bound against low-degree polynomials which extends the known lower bounds for both sparse PCA and tensor PCA, leading to a more intricate understanding of how all three parameters ($n$, $k$ and $p$) interact.

\subsection{Related work}

Disregarding computational efficiency, it is easy to see that optimal statistical guarantees can be achieved with a simple exhaustive search (corresponding to the maximum likelihood estimator): find a $k$-sparse unit vector maximizing $\iprod{\pmb Y, x^{\otimes p}}$.
This algorithm returns a $k$-sparse unit vector $\widehat{x}$ achieving constant squared correlation\footnote{One could also aim to find a unit vector with correlation approaching one or, in the restricted setting of $x \in \set{0,\pm 1/\sqrt{k}}$, aim to recover the support of $x$. At the coarseness of our discussion here, these goals could be considered mostly equivalent.} with the signal $x$ as soon as $\lambda \gtrsim \sqrt{k \cdot \log (np/k)}$.
That is, whenever $\lambda \gtrsim \max_{\Norm{x}=1, \Norm{x}_0=k} \iprod{\pmb W, \tensorpower{x}{p}}$.
Unfortunately, this approach runs in time exponential in $k$ and takes super-polynomial time when $p \lesssim k$.\footnote{Note that the problem input is of size $n^p$. So when $p \gtrsim k$, exhaustive search takes $n^{\cO(p)}$ time which is polynomial in $n^p$. Thus, the interesting parameter regimes occur when $p \lesssim k$.}
\emph{As such, we assume $p \leq k$ from now on.} %in the rest of this paper.}

Taking into account computational aspects, the picture changes.
A good starting point to draw intuition for \ref{eq:generic-sparse-tensor-pca} is the literature on sparse PCA and tensor PCA.
We briefly outline some known results here.
To simplify the discussion, we hide absolute constant multiplicative factors using $\cO(\cdot)$, $\Omega(\cdot)$, $\lesssim$, and $\gtrsim$, and hide multiplicative factors logarithmic in $n$ using $\widetilde{\cO}(\cdot)$.

\subsubsection{Sparse PCA (Wigner noise)}
Sparse PCA with Wigner noise exhibits a sharp phase transition in the top eigenvalue of $\pmb Y$ for $\lambda \geq \sqrt{n}$ \cite{feral2007largest}.
In this strong signal regime, the top eigenvector\footnote{By ``top eigenvector'' or ``leading eigenvector'', we mean the eigenvector corresponding to ``largest (in absolute value) eigenvalue''.} 
 $v$ of $\pmb Y$ correlates\footnote{More precisely, the vector consisting of the $k$ largest (in absolute value) entries of $v$.} with $x$ with high probability, thus the following spectral method achieves the same guarantees as the exhaustive search suggested above: compute a leading eigenvector of $\pmb Y$ and restrict it to the top $k$ largest entries in absolute value.
Conversely, when $\lambda < \sqrt{n}$, the top eigenvector of $\pmb Y$ does not correlate with the signal $x$.
In this weak signal regime, \cite{johnstone2009consistency} proposed a simple algorithm known as diagonal thresholding: compute the top eigenvector of the principal submatrix defined by the $k$ largest diagonal entries of $\pmb Y$.
This algorithm recovers the sparse direction when $\lambda \gtrsim \widetilde{\cO}(k)$, thus requiring almost an additional $\sqrt{k}$ factor when compared to inefficient algorithms.
More refined polynomial-time algorithms (low-degree polynomials \cite{dOrsi2020}, covariance thresholding \cite{deshpande2016sparse} and the basic SDP relaxation \cite{DBLP:journals/siamrev/dAspremontGJL07, dOrsi2020}) only improve over diagonal thresholding by a logarithmic factor in the regime $n^{1-o(1)} \lesssim k^2 \lesssim n$.
Interestingly, multiple results suggest that this \emph{information-computation gap} is inherent to the sparse PCA problem \cite{berthet2013complexity,berthet2013optimal, ding2019subexponential,dOrsi2020}.
Subexponential time algorithms and lower bounds have also been shown.
For instance, \cite{ding2019subexponential, DBLP:conf/colt/HoltzmanSV20} presented smooth trade-offs between signal strength and running time.\footnote{
Both works studied the single spike matrix setting.
\cite{DBLP:conf/colt/HoltzmanSV20} only considers the Wishart noise model and thus its guarantees cannot be compared to ours.
\cite{ding2019subexponential} studied both the Wishart and Wigner noise models.
In the Wishart noise model setting, both \cite{DBLP:conf/colt/HoltzmanSV20} and \cite{ding2019subexponential} observe the same tradeoff between running time and signal-to-noise ratio.
In the Wigner noise model setting, our algorithm and the algorithm of \cite{ding2019subexponential} offer the same smooth-trade off between running time and signal strength, up to universal constants.
}

\subsubsection{Tensor PCA}
In tensor settings, computing $\max_{\Norm{x}=1}\iprod{\pmb Y, \tensorpower{x}{p}}$ is NP-hard already for $p=3$ \cite{hillar2013most}.
For even tensor powers $p$, one can unfold the tensor $\pmb Y$ into a $n^{p/2}$-by-$n^{p/2}$ matrix and solve for the top eigenvector \cite{richard2014statistical}.
However, this approach is sub-optimal for odd tensor powers.
For general tensor powers $p$, a successful strategy to tackle tensor PCA has been the use of semidefinite programming \cite{hopkins2015tensor, bhattiprolu2016sum, hopkins2019robust}.
Spectral algorithms inspired by the insight of these convex relaxations have also been successfully applied to the problem \cite{schramm2017fast}.
These methods succeed in recovering the single-spike $x$ when $\lambda \gtrsim \widetilde{\cO}\Paren{n^{p/4}}$, thus exhibiting a large gap when compared to exhaustive search algorithms.
Matching lower bounds have been shown for constant degrees in the Sum-of-Squares hierarchy \cite{bhattiprolu2016sum,hopkins2017power} and through average case reductions \cite{DBLP:conf/colt/BrennanB20}.

\subsubsection{Sparsity-exploiting algorithms and tensor algorithms}
It is natural to ask how do the characteristics of sparse PCA and tensor PCA extend to the more general setting of \ref{eq:generic-sparse-tensor-pca}.
In particular, there are two main observations to be made.

The first observation concerns the sharp computational transition that we see for $k\lesssim \sqrt{n}$ in sparse PCA.
In these highly sparse settings, the top eigenvector of $\pmb Y$ does not correlate with the signal $x$ and so algorithms primarily based on spectral methods fail to recover it.
Indeed, the best known guarantees are achieved through algorithms that crucially exploit the sparsity of the hidden signal.
These algorithms require the signal strength to satisfy $\lambda \geq \widetilde{O}( \sqrt{k})$, with only logarithmic dependency on the ambient dimension.
To exemplify this to an extreme, notice how the following algorithm can recover the support of $\dyad{x}$ with the same guarantees as diagonal thresholding, essentially disregarding the matrix structure of the data: zero all but the $k^2$ largest (in absolute value) entries of $\pmb Y$.
A natural question to ask is whether a similar phenomenon may happen for higher order tensors.
In the highly sparse settings where $k\lesssim \sqrt{n}$, \textit{can we obtain better algorithms exploiting the sparsity of the hidden vector?}
Recently, a partial answer appeared in \cite{DBLP:journals/corr/abs-2005-10743} with a polynomial time algorithm recovering the hidden signal for $\lambda \geq \widetilde{O}(p\cdot k^{p/2})$, albeit with suboptimal dependency on the tensor order $p$.

The second observation concerns the computational-statistical gap in the spiked tensor model.
As $p$ grows, the gap between efficient algorithms and exhaustive search widens with the polynomial time algorithms requiring signal strength $\lambda \geq \widetilde{O}\Paren{n^{p/4}}$ while exhaustive search succeeds when $\lambda \geq \widetilde{O}(\sqrt{n})$ \cite{richard2014statistical}.
The question here is: \emph{how strong is the dependency on $p$ for efficient algorithms in sparse signal settings?}

In this work, we investigate these questions in the high order tensors regime $p \in \omega(1)$. We present a family of algorithms with a smooth trade-off between running time and signal-to-noise ratio.
Even restricting to polynomial-time settings, our algorithms improve over previous results. Furthermore, through the lens of low-degree polynomials, we provide rigorous evidence of an \textit{exponential gap} in the tensor order $p$ between algorithms and lower bounds.

\begin{remark*}
The planted sparse densest sub-hypergraph model \cite{corinzia2019exact, buhmannrecovery, corinzia2020statistical} is closely related to \ref{eq:generic-sparse-tensor-pca}.
We discuss this model in \cref{sec:PSDM}.
\end{remark*}

\subsection{Results}\label{sec:results}

\subsubsection{Single spike setting}
Consider first the restricted, but representative, case where the planted signal is a $(k,A)$-sparse unit vector with $k$ non-zero entries having magnitudes in the range $\Brac{ \frac{1}{A \sqrt{k}}, \frac{A}{\sqrt{k}} }$ for some constant $A \geq 1$.
We say that the signal is \emph{flat} when $A = 1$ and \emph{approximately flat} when $A \geq 1$.

Our first result is a limited brute force algorithm -- informally, an algorithm that smoothly interpolates between some brute force approach and some ``simple'' polynomial time algorithm -- that \emph{exactly} recovers the signal support of the planted signal\footnote{A similar algorithm was analyzed by \cite{ding2019subexponential} for the special case of $p=2$ and $r=1$.}.

\begin{theorem}[Algorithm for single spike sparse tensor PCA, Informal]
	\label{thm:main-algorithm-informal}
	Let $A \geq 1$ be a constant. Consider the observation tensor
	\[
	\pmb Y = \pmb W + \lambda \tensorpower{x}{p}
	\]
	where the additive noise tensor $\pmb W \in \otimes^p\R^n$ contains i.i.d.\ $N(0,1)$ entries and the signal $x\in \R^n$ is a $(k,A)$-sparse unit vector with signal strength $\lambda > 0$.
	Let $1 \leq t \leq k$ be an integer.
	Suppose that
	\[
	\lambda \gtrsim \sqrt{ t \left( \frac{2 A^2 k}{t} \right)^p \ln n}\,.
	\]
	Then, there exists an algorithm that runs in $\cO(p n^{p + t})$ time and, with probability 0.99, outputs the support of $x$. 
\end{theorem}

Let's first consider \cref{thm:main-algorithm-informal} in its simplest setting where $A=1$ and $t$ is a fixed constant.
For $k\lesssim \sqrt{n}$, the theorem succeed when $\lambda \geq \widetilde{\cO}(k^{p/2})$, thus improving over the guarantees of known tensor PCA methods which require $\lambda \geq \widetilde{\cO}(n^{p/4})$. 
In addition, since support recovery is \emph{exact}, one can obtain a good estimate\footnote{Recovery is up to a global sign flip since $\iprod{u,v}^p = \iprod{u,-v}^p$ for even tensor powers $p$.} of the planted signal by running any known tensor PCA algorithm on the subtensor corresponding to its support.
Indeed, the resulting subtensor will be of significantly smaller dimension and the requirement needed on the signal strength by single-spike tensor PCA algorithms are weaker than the requirement we impose on $\lambda$ (see \cref{remark:reconstruct} for details).
As a result, our algorithm recovers the guarantees of diagonal thresholding in the matrix ($p=2$) setting. 
Our polynomial-time algorithm also improves over the result of \cite{DBLP:journals/corr/abs-2005-10743}, which required $\lambda \gtrsim \sqrt{p k^p\log n}$, by removing the polynomial dependency of the tensor order $p$ in the signal strength $\lambda$.\footnote{The result of \cite{DBLP:journals/corr/abs-2005-10743} extends to the settings where $\pmb Y = \pmb W + \lambda \mathcal{X}$ for an approximately flat tensor $\mathcal{X} \in \otimes^p \R^n$. Both \cref{thm:main-algorithm-informal} and \cref{thm:main-algorithm-multi-informal} can also be extended to these settings (see \cref{sec:general-tensor-extension}).}

Consider now the limited brute force parameter $t$.
From the introductory exposition, we know that one can obtain a statistically optimal algorithm by performing a brute force search over the space of $k$-sparse flat vectors in $\R^n$.
The \emph{limited brute force} algorithm is a natural extension that takes into account computational constraints by searching over the smaller set of $t$-sparse flat vectors, for $1 \leq t \leq k$, to maximize $\iprod{\pmb Y, \tensorpower{u}{p}}$.
The parametric nature of the algorithm captures both the brute force search algorithm (when $t = k$) and the idea of diagonal thresholding (when $t=1$ and $p=2$).
As long as $t \leq k$, using a larger $t$ represents a direct trade-off between running time and the signal-to-noise ratio.
Extending the result to approximately flat vectors, the dependency on $A$ in the term $\Paren{2 A^2}^{p}$ can be removed by increasing the computational budget to some value $t' \geq 2 A^2 t$.

\subsubsection{Multiple spikes} 

\begin{theorem}[Algorithm for multi-spike sparse tensor PCA, Informal]
	\label{thm:main-algorithm-multi-informal}
	Let $A \geq 1$ be a constant.  Consider the observation tensor
    \[
    \pmb Y = \pmb W + \sum_{q=1}^r \lambda_q \tensorpower{x_{(q)}}{p}
    \]
    where the additive noise tensor $\pmb W \in \otimes^p\R^n$ contains i.i.d.\ $N(0,1)$ entries and the signals $x_{(1)}, \ldots, x_{(r)}\in \R^n$ are $(k,A)$-sparse unit vectors with disjoint supports and corresponding signal strengths $\lambda_1 \geq \ldots \geq \lambda_r > 0$.
    Let $1 \leq t \leq k$ be an integer and $0 < \eps \leq 1/2$.
    Suppose that
	\begin{align*}
		\lambda_r \gtrsim \frac{1}{\eps} \cdot \sqrt{ t \Paren{\frac{2 A^2 k}{t}}^p \ln n} \quad \text{and} \quad \lambda_r \gtrsim A^{2p} \cdot \Paren{2 \eps}^{p-1} \cdot \lambda_1\,.
	\end{align*}
	Then, there exists an algorithm that runs in $\cO(rpn^{p + t})$ time and, with probability 0.99, outputs the individual signal supports of $x_{(\pi(1))}, \ldots, x_{(\pi(r))}$ for some unknown bijection $\pi: [r] \rightarrow [r]$.
\end{theorem}

\cref{thm:main-algorithm-multi-informal} requires two assumptions on the signals: (1) signals have disjoint support; (2) there is a bounded signal strength gap of $\lambda_r \gtrsim A^{2p} \cdot \Paren{2 \eps}^{p-1} \cdot \lambda_1$.
In the context of sparse PCA, algorithms that recover multiple spikes (e.g.\ \cite{johnstone2009consistency,deshpande2016sparse}) only require the sparse vectors to be orthogonal.
However, their guarantees are of the form $\lambda_r \geq \widetilde{\cO}\Paren{\Card{\bigcup_{q\in [r]}\supp\Paren{x_{(q)}}}}$.
That is, when the $r$ signals have disjoint supports, they require the smallest signal to satisfy $\lambda_r \geq \widetilde{\cO}\Paren{k \cdot r}$.
In comparison, already for constant $t$, \cref{thm:main-algorithm-multi-informal} successfully recovers the supports when $\lambda_r\geq \widetilde{\cO}(k)$, thus removing the dependency on the number of signals and improving the bound by a $1/r$ factor\footnote{
It is an intriguing question whether an improvement of $1/r$ can be achieved in the more general settings of orthogonal spikes. Our approach relies on the signals having disjoint support and we expect it to \emph{not} be generalizable to orthogonal signals. This can be noticed in the simplest settings with brute-force parameter $t=1$ and $p=2$ where the criteria of \cref{alg:multi-spike} for finding an entry of a signal vector is to look at the diagonal entries of the data matrix. In this case, the algorithm may be fooled since the largest diagonal entry can depend on more than one spike. Nevertheless, we are unaware of any fundamental barrier suggesting that such guarantees are computationally hard to achieve.
}.
Meanwhile, the bounded signal strength gap assumption is a common identifiability assumption (e.g.\ see \cite{cai2013sparse, deshpande2016sparse}). 
We remark that \cref{thm:main-algorithm-multi-informal} provides a tradeoff between this signal strength gap assumption and the signal strengths: we can recover the supports with a smaller gap if the signal strengths are increased proportionally -- increasing $\lambda_r$ by a multiplicative factor $\alpha$ enables the algorithm to succeed with gap that is smaller by a multiplicative factor of $1/\alpha$.
As an immediate consequence, we also obtain a tradeoff between gap assumption and running time: every time we double $t$ (while ensuring $1 \leq t \leq k$), $\lambda_r$ increases by a factor of $(1/\sqrt{2})^{p-1}$ and thus the algorithm can succeed with a smaller gap.
Finally, as in the single spike case, the exact support recovery allow us to obtain good estimate of each signal by running known tensor PCA algorithms.

\paragraph{Remark} We remark that these results can be extended to the general tensor settings
\[
	\pmb Y = \pmb W + \sum_{q=1}^r \lambda_q \mathcal{X}_{(q)}
\]
where for $q\in [r] $, $\mathcal{X}_{(q)}=x_{(q,1)}\otimes\cdots\otimes x_{(q,p)}\in \otimes^p\R^n$  in a natural way.
See \cref{sec:general-tensor-extension}.

\subsubsection{An exponential gap between lower bounds and algorithms} 
\ref{eq:generic-sparse-tensor-pca} generalizes both sparse PCA and tensor PCA.
Hence, a tight hardness result for the model is interesting as it may combine and generalize the known bounds for these special cases.
Here, we give a lower bound for the restricted computational model captured by \emph{low-degree polynomials}.
Originally developed in the context of the sum of squares hierarchy, this computational model appears to accurately predict the current best-known guarantees for problems such as sparse PCA, tensor PCA, community detection, and planted clique (e.g.\ see \cite{hopkins2017efficient, hopkins2017power, hopkins2018statistical, barak2019nearly, ding2019subexponential, kunisky2019notes, dOrsi2020}).

\begin{theorem}[Lower bound for low-degree polynomials, Informal]\label{thm:main-lower-bound-informal}
	Let $1 \leq D \leq 2n/p$ and $\nu$ be the distribution of $\pmb Z \in \otimes^p \R^n$ with i.i.d.\ entries from $N(0,1)$.
	Then, there exists a distribution $\mu$ over tensors $\pmb Y \in \otimes^p \R^n$ of the form
	\[
	\pmb Y = \pmb W + \lambda \tensorpower{\pmb x}{p}
	\]
	where $\pmb W \in \otimes^p \R^n$ is a noise tensor with i.i.d.\ $N(0,1)$ entries, the marginal distribution of $\pmb x$ is supported on vectors with entries $\Set{\pm 1/\sqrt{k},0}^n$, and $\pmb x$ and $\pmb W$ are distributionally independent, such that whenever
	\[
	\lambda \lesssim \frac{\sqrt{D}}{2^p} \min \left\{ \Paren{ \frac{n}{pD} }^{p/4},\; \Paren{ \frac{k}{pD} \Paren{ 1 + \Abs{\ln \Paren{ \frac{npD}{ek^2} }}}}^{p/2} \right\},
	\]
	$\mu$ is indistinguishable\footnote{In the sense that for any low-degree polynomial $p(\pmb Y)$ we have $\frac{\E_\mu p(\pmb Y)-\E_\nu p(\pmb Y)}{\sqrt{\Var p(\pmb Y)}} \in o(1)$. See \cref{sec:background-low-degree-method}.} from $\nu$
	with respect to all polynomials of degree at most $D$.
\end{theorem}

\cref{thm:main-lower-bound-informal} states that for certain values of $\lambda$, low-degree polynomials cannot be used to distinguish between the distribution of $\pmb{Y}$ and $\pmb {W}$ as typical values of low-degree polynomials are the same (up to a vanishing difference) under both distributions. 
The theorem captures known results in both sparse and tensor PCA settings.
When $p=2$, our bound reduces to $\lambda \lesssim \min \left\{ \sqrt{n},\; \frac{k}{\sqrt{D}} \Paren{ 1 + \Abs{\ln \Paren{ \frac{2nD}{ek^2} }}} \right\}$,  matching known low-degree bounds of \cite{ding2019subexponential} in the sparse PCA setting.
Meanwhile, in the tensor PCA settings ($p\geq 2$, $k=n$), \cref{thm:main-lower-bound-informal} implies a bound of the form $\lambda \lesssim \sqrt{D} \Paren{\frac{n}{pD}}^{p/4}$, thus recovering the results of \cite{kunisky2019notes}.

For constant power $p$ and $k \lesssim\sqrt{n}$, our lower bound suggests that no estimator captured by polynomials of degree $D \lesssim t \log n$ can improve over our algorithmic guarantees by more than a logarithmic factor.
However, for $p\in \omega(1)$, an exponential gap appears between the bounds of \cref{thm:main-lower-bound-informal} and state-of-the-art algorithms (both in the sparse settings as well as in the dense settings).\footnote{In particular, in the sparse settings $k \leq \sqrt{np}$, the $p^{-p/2}$ factor could not be seen in the restricted case of sparse PCA (as this factor is a constant when $p=2$).}
As a concrete example, let us consider the setting where $p = n^{0.1} < k$.
The polynomial time algorithm of \cref{thm:main-algorithm-informal} requires $\lambda\geq \tilde{\cO}(k^{p/2})$ while according to \cref{thm:main-lower-bound-informal} it may be enough to have $\lambda \geq \tilde{\cO} \Paren{k/n^{0.1}}^{p/2}$.
Similarly, for $k\gtrsim \sqrt{np}$, known tensor algorithms recovers the signal for $\lambda \geq \tilde{\cO}(n^{p/4})$ while our lower bound only rules out algorithms for $\lambda \leq \tilde{\cO}\Paren{n^{0.9 \cdot p/4}}$.

Surprisingly, for the distinguishing problem considered in \cref{thm:main-lower-bound-informal}, these bounds appear to be tight.
For a wide range of parameters (in both the dense and sparse settings) there exists polynomial time algorithms that can distinguish the distributions $\nu$ and $\mu$ right at the threshold considered in \cref{thm:main-lower-bound-informal} (see \cref{sec:appendix-computational-bounds}).
It remains a fascinating open question whether sharper recovering algorithms can be designed or stronger lower bounds are required.

Finally, we would like to highlight that this non-trivial dependency on $p$ is a purely computational phenomenon as it does not appear in information-theoretic bounds (see \cref{sec:appendix-info-lower-bound}).

\paragraph{Remark}
Note that \cref{thm:main-lower-bound-informal} is \emph{not} in itself a lower bound for the recovery problem.
However,  any algorithm which obtains a good estimation of the signal vector $x$ for signal strength $\lambda\geq \sqrt{k\log n}$ can be used to design a probabilistic algorithm which solve the distinguishing problem for signal strength $\cO_p(\lambda)$.
Let us elaborate.
Consider an algorithm that given $\pmb Y = \pmb W + \lambda x^{\otimes p}$ outputs a vector $\hat{x}$ such that $\Abs{\iprod{\hat{x}, x}} \geq 0.9$.
With high probability, $\max_{\Abs{z}_2 = 1, \Abs{z}_0 = k} \Abs{\iprod{\pmb W, z^{\otimes p} }} \leq \widetilde{\cO}(\sqrt{k})$ and thus $\Abs{\iprod{ \pmb Y, \hat{x}^{\otimes p} }} \geq \lambda \cdot (0.9)^p - \widetilde{\cO}(\sqrt{k})$.
Therefore, one can solve the distinguishing problem as follows: output ``planted'' if $\Abs{\iprod{ \pmb Y, \hat{x}^{\otimes p} }} \gtrsim \sqrt{k \log n}$ and ``null'' otherwise.

\subsection{Notation and outline of paper}

\paragraph{Notation}
We write random variables in boldface and the set $\{1, \ldots, n\}$ as $[n]$.
We hide absolute constant multiplicative factors and multiplicative factors logarithmic in $n$ using standard notations: $\cO(\cdot)$, $\Omega(\cdot)$, $\lesssim$, $\gtrsim$, and $\widetilde{\cO}(\cdot)$. We denote by $e_1,\ldots,e_n\in \R^n$ the standard basis vectors.
For $x \in \R^n$, we use $\supp(x) \subseteq [n]$ to denote the set of support coordinates.
We say that $x$ is a \emph{$(k,A)$-sparse vector} if $k \in [n]$, constant $A \geq 1$, $\Abs{\supp(x)} = k$, and $\frac{1}{A \sqrt{k}} \leq \Abs{x_\ell} \leq \frac{A}{\sqrt{k}}$ for $\ell \in \supp(x)$.
When $A=1$, we say that $x$ is a \emph{$k$-sparse flat vector} and may omit the parameter $A$.
For general $A \geq 1$, we say that $x$ is \emph{approximately flat}.
For an integer $t \geq 1$, we define $U_t = \Set{u \in \Set{-\frac{1}{\sqrt{t}}, 0, \frac{1}{\sqrt{t}}}^n : \Abs{\supp(u)} = t}$
as the set of $t$-sparse flat vectors.
For a tensor $T \in \otimes^p \R^n$ and a vector $u \in \R^n$, their inner product is defined as
$\iprod{T, \tensorpower{u}{p}} = \underset{i_1,\ldots,i_p\in [n]}{\sum} T_{i_1,\ldots,i_p} u_{i_1} \ldots u_{i_p}$.

\paragraph{Outline}
The rest of the paper is organized as follows:
In \cref{sec:techniques-recovery}, we introduce the main ideas behind \cref{thm:main-algorithm-informal} and \cref{thm:main-algorithm-multi-informal}.
\cref{sec:appendix_preliminaries} contains preliminary notions.
We formally prove \cref{thm:main-algorithm-informal} and \cref{thm:main-algorithm-multi-informal} in \cref{sec:appendix-limited-brute-force}.
The lower bound \cref{thm:main-lower-bound-informal} is given in \cref{sec:appendix-computational-bounds}.
We present an information theoretic bound in \cref{sec:appendix-info-lower-bound}.
\cref{sec:PSDM} discusses the planted sparse densest sub-hypergraph model.
Finally, \cref{sec:missing-proofs} contains technical proofs required throughout the paper.
\section{Recovering signal supports via limited brute force searches}
\label{sec:techniques-recovery}

We describe here the main ideas behind our limited brute force algorithm.
We consider the model
\begin{restatable}[Sparse spiked tensor model]{model}{modeldefn}
	\label{def:sstm}
	For $A\geq 1, r\geq 1, k\leq n$ we observe a tensor of the form
	\[
	\pmb Y = \pmb W + \sum_{q=1}^r \lambda_q \tensorpower{x_{(q)}}{p} \in \otimes^p\R^n
	\]
	where $\pmb W \in \otimes^p\R^n$ is a noise tensor with i.i.d.\ $N(0,1)$ entries, $\lambda_1 \geq \ldots \geq \lambda_r > 0$ are the signal strengths, and $x_{(1)}, \ldots, x_{(r)}$ are $k$-sparse flat unit length signal vectors.
\end{restatable}
We first look at the simplest setting of a single flat signal (i.e.\ $A=1$ and $r=1$), which already capture the complexity of the problem while also sparing many details.
Second, we explain how to extend the analysis to multiple flat signals (i.e.\ $A=1$ and $r \geq 1$).
Third, we consider approximately sparse vectors.
For a cleaner discussion, we assume here that all the non-zero entries of the sparse vector $x$ and vectors in the set $U_t$ have positive sign.
Our techniques also extend\footnote{We provide details for this extension in \cref{sec:general-tensor-extension}.} to general signal tensors $x_{(1)}\otimes\cdots \otimes x_{(p)}\in \otimes^p\R^n$.

\subsection{Single flat signal}
\label{sec:techniques-single-flat} 

As already mentioned in the introduction, a brute force search over $ U_k$ for the vector maximizing $\iprod{\pmb Y, \tensorpower{u}{p}}$ returns the signal vector $x$ (up to a global sign flip) with high probability whenever $\lambda \gtrsim \sqrt{k\log n}$. This algorithm provides provably optimal guarantees but requires exponential time (see \cref{sec:appendix-info-lower-bound} for an information-theoretic lower bound).
The idea of a \emph{limited brute force search} is to search over a smaller set $U_t$ ($1 \leq t \leq k$) instead, and use the maximizer $\pmb v_*$ to determine the signal support $\supp\Paren{x}$.
The hope is that for a sufficiently large signal-to-noise ratio, this $t$-sparse vector $\pmb v_*$ will still be non-trivially correlated with the hidden vector $x$.
Indeed as $t$ grows, the requirement on $\lambda$ decreases towards the information-theoretic bound, at the expense of increased running time.

As a concrete example, consider the matrix settings $(p=2)$. It is easy to generalize the classic diagonal thresholding algorithm (\cite{johnstone2009consistency}) into a limited brute-force algorithm.
Recall that diagonal thresholding identifies the support of $x$ by picking the indices of the largest $k$ diagonal entries of $\pmb Y$. In other words, the algorithm simply computes $\iprod{\pmb Y, \tensorpower{e_i}{2}}$ for all $i\in [n]$ and returns the largest $k$ indices.
From this perspective, the algorithm can be naturally extended to $t > 1$ by computing the $\binom{k}{t}$ vectors $u \in U_t$ maximizing $\iprod{\pmb Y,\tensorpower{u}{2}}$ and reconstructing the signal from them.
For $t=k$, the algorithm corresponds to exhaustive search.\bigskip

With this intuition in mind, we now introduce our family of algorithms, heavily inspired by \cite{ding2019subexponential}.
We first apply a preprocessing step to obtain two independent copies of the data. 

\begin{algorithm}[H]
	\caption{Preprocessing}
	\label{alg:preprocessing}
	\begin{algorithmic}
		\State \textbf{Input}: $\pmb Y$.
		\State Sample a Gaussian tensor $\pmb Z \in \otimes^p \R^n$ where each entry is an i.i.d.\ standard Gaussian $N(0,1)$.
		\State \textbf{Return} two independent copies $\pmb Y^{(1)}$ and $\pmb Y^{(2)}$ of $\pmb Y$ as follows:
		\[
		\pmb Y^{(1)} = \frac{1}{\sqrt{2}} \Paren{\pmb Y + \pmb Z}
		\quad \text{and} \quad
		\pmb Y^{(2)} = \frac{1}{\sqrt{2}} \Paren{\pmb Y - \pmb Z}
		\]
	\end{algorithmic}
\end{algorithm}

\cref{alg:preprocessing} effectively creates two independent copies of the observation tensor $\pmb Y$.
To handle the noise variance, the signal-to-noise ratio is only decreased by the constant factor $1/\sqrt{2}$.
For simplicity, we will ignore this constant factor in the remainder of the section and leave the formalism to the appendix.

\begin{algorithm}[H]
	\caption{Single spike limited brute force}
	\label{alg:basic}
	\begin{algorithmic}
		\State \textbf{Input}: $k, t$ and $\pmb Y^{(1)}, \pmb Y^{(2)}$ obtained from \cref{alg:preprocessing}.
	
		\State Compute $\pmb v_*:=\argmax_{u \in U_t}\iprod{\pmb Y^{(1)}, \tensorpower{u}{p}}$.
		\State Compute the vector $\pmb \alpha\in \R^n$ with entries  $\pmb \alpha_\ell := \iprod{\pmb Y^{(2)}, \tensorpower{\pmb v_*}{p-1}\otimes e_\ell}$ for every $\ell \in [n]$.
		\State \textbf{Return} the indices of the largest $k$ entries of $\pmb \alpha$.
	\end{algorithmic}
\end{algorithm}

The signal support recovery process outlined in \cref{alg:basic} has two phases.
In the first phase, we search over $U_t$ to obtain a vector $\pmb v_*$ that is correlated with the signal $x$.
In the second phase, we use $\pmb v_*$ to identify $\supp(x)$.
The correctness of the algorithm follows from these two claims:
\begin{enumerate}[(i)]
	\item The $t$-sparse maximizer $\pmb v_*$ shares a large fraction of its support coordinates with signal $x$.
	\item The $k$ largest entries of $\pmb \alpha$ belong to the support $\supp(x)$ of signal $x$.
\end{enumerate}

Crucial to our analysis is the following standard concentration bound on Gaussian tensors when interacting with $t$-sparse unit vectors.
We directly use \cref{lem:informal-concentration-gaussian} in our exposition here, and formally prove a more general form in \cref{sec:sparse-norm-bounds-proof}.
\begin{lem}
	\label{lem:informal-concentration-gaussian}
	Let $p\leq n$, $t>0$ be an integer, and $\pmb W\in \otimes^p\R^n$ be a tensor with i.i.d.\ $N(0,1)$ entries.
	Then, with high probability, for any $u \in U_t$,
	\begin{align*}
		\iprod{\pmb W, \tensorpower{u}{p}}&\lesssim \sqrt{t \log n}\,.
	\end{align*}
\end{lem}
\bigskip

For some constant $0 < \eps \leq 1/2$, suppose that
\begin{equation}
\label{eq:lambda-value}
\lambda \gtrsim \frac{1}{\eps \cdot (1-\eps)^{p-1}} \cdot \sqrt{t \Paren{\frac{k}{t}}^p \log n}\,.
\end{equation}

For any $u \in U_t$ with support $\supp(u) \subseteq \supp(x)$, we have
\[
\iprod{\pmb Y^{(1)}, \tensorpower{u}{p}}
= \lambda \iprod{x,u}^p+ \iprod{\pmb W^{(1)}, \tensorpower{u}{p}}
\geq \lambda \cdot \Paren{\frac{t}{k}}^{\frac{p}{2}}  - \cO\Paren{\sqrt{t\log n}}\,.
\]

On the other hand, any $u \in U_t$ with support satisfying $\Card{\supp(u)\cap \supp(x)} \leq (1-\eps)\cdot t$ has small correlation with $\pmb Y^{(1)}$ in the sense that
\begin{align*}
	\iprod{\pmb Y^{(1)}, \tensorpower{u}{p}}
	= \lambda \iprod{x,u}^p+\iprod{\pmb W^{(1)}, \tensorpower{u}{p}}
	\leq \lambda \cdot (1-\eps)^{p} \cdot \Paren{\frac{t}{k}}^{\frac{p}{2}} + \cO\Paren{\sqrt{t\log n}}\,.
\end{align*}

By \cref{eq:lambda-value}, with high probability, $\pmb v_*$ will have at least a fraction $(1-\eps)$ of the support contained in $\supp(x)$, yielding the first claim.
Observe that $\pmb v_*$ does not completely overlap with $x$.
A priori, this might seem to be an issue.
However, it turns out that we can still use $\pmb v_*$ to exactly reconstruct the support of $x$.
Indeed, for all $\ell \in \supp(x)$,
\begin{align*}
	\pmb \alpha_\ell
	&= \lambda\cdot x_\ell\cdot \iprod{x, \pmb v_*}^{p-1}	+  \iprod{\pmb W^{(2)}, \tensorpower{\pmb  v_*}{p-1}\otimes e_\ell} \\
	&\geq \lambda\cdot \frac{(1-\eps)^{p-1}}{\sqrt{k}}\cdot \Paren{\frac{t}{k}}^{\frac{p-1}{2}}+  \iprod{\pmb W^{(2)}, \tensorpower{\pmb  v_*}{p-1}\otimes e_\ell} \\
	&\gtrsim \frac{1}{\eps} \cdot \sqrt{\log n} +  \iprod{\pmb W^{(2)}, \tensorpower{\pmb  v_*}{p-1}\otimes e_\ell} \,.
\end{align*}
Now, by independence of $\pmb W^{(2)}$ and $\pmb v_*$, $\iprod{\pmb W^{(2)}, \tensorpower{\pmb  v_*}{p-1}\otimes e_\ell}$ behaves like a standard Gaussian.
Thus, with high probability, $ \Abs{\iprod{\pmb W^{(2)}, \tensorpower{\pmb  v_*}{p-1}\otimes e_\ell}}\lesssim\sqrt{ \log n}$ and
$
\pmb \alpha_\ell \gtrsim \sqrt{\log n}
$.
Conversely, if $\ell$ is \emph{not} in the support of the signal, then $\pmb \alpha_\ell \lesssim \sqrt{\log n}$.
So, the vector $\pmb \alpha$ acts as indicator of the support of $x$!

\begin{remark}
	\textnormal{
		In its simplest form of $t=1$, \cref{alg:basic} does not exploit the tensor structure of the data: it performs entry-wise search for the largest (in magnitude) over a subset of $\pmb Y$.
		However, this is no longer true as $t$ grows.
		For $t=k$, the algorithm computes the $k$-sparse flat unit vector $u$ maximizing $\iprod{\pmb Y^{(1)}, \tensorpower{u}{p}}$.
	}
\end{remark}

\subsection{Multiple flat signals with disjoint signal supports}
\label{sec:techniques-multi-flat}

Consider now the setting with $r>1$ spikes.
Recall that we assumed the vectors $x_{(1)}, \ldots, x_{(r)}$ to have non-intersecting supports.
We also assumed that for any $q, q' \in [r]$ and some fixed scalar $0 \leq \kappa \leq 1$, if $\lambda_q \geq \lambda_{q'}$, then $\lambda_{q'} \geq \kappa \cdot \lambda_q$.
We remark that we may not recover the signal supports in a known order, but we are guaranteed to recover \emph{all of them} exactly.
For simplicity of discussion, let us assume here that we recover the vector $x_{(i)}$ at iteration $i$.

The idea to recover the $r$ spikes is essentially to run \cref{alg:basic} $r$ times.
At first, we compute the $t$-sparse vector $\pmb v_*$ by maximizing the product $\iprod{\pmb Y^{(1)}, \tensorpower{\pmb v_*}{p}}$.
Then, using $\pmb v_*$, we compute the vector $\pmb \alpha$ to obtain a set $\cI_1 \subseteq [n]$.
With high probability, we will have $\cI_1= \supp\Paren{x_{(1)}}$ and so we will exactly recover the support of $x_{(1)}$.
In the second iteration of the loop, we repeat the same procedure with the additional constraint of searching only over the $n-k$ dimensional subset of $U_t$ containing vectors with disjoint support from $\cI_1$.
Similarly, at iteration $i$, we search over the subset of $U_t$ containing vectors with disjoint support from $\underset{1 \leq j < i}{\bigcup} \cI_j$.
As before, we first preprocess the data to create two independent copies $\pmb Y^{(1)}$ and $\pmb Y^{(2)}$.
Concretely:

\begin{algorithm}[H]
	\caption{Multi-spike limited brute force}
	\label{alg:multi-spike}
	\begin{algorithmic}
		\State \textbf{Input}: $k, t, r$ and $\pmb Y^{(1)}, \pmb Y^{(2)}$ obtained from \cref{alg:preprocessing}.
		\State \textbf{Repeat} for $i=1$ to $r$:
		\State $\qquad\quad $  Compute $\pmb v_*:=\argmax_{u \in U_t} \iprod{\pmb Y^{(1)}, \tensorpower{u}{p}}$ subject to $\supp\Paren{\pmb v_*} \cap \Paren{\underset{1 \leq j < i}{\bigcup} \cI_j} = \emptyset$.
		\State $\qquad\quad $ Compute the vector $\pmb \alpha\in \R^n$ with entries  $\pmb \alpha_\ell := \iprod{\pmb Y^{(2)}, \tensorpower{\pmb v_*}{p-1} \otimes e_\ell }$ for every $\ell \in [n]$.
		\State $\qquad\quad $  Let $\cI_i$ be the set of indices of the largest $k$ entries of $\pmb \alpha$.
		\State \textbf{Return} $\cI_{1},\ldots,\cI_r$.
	\end{algorithmic}
\end{algorithm}
%\end{restatable}

The proof structure is similar to that of \cref{alg:basic} and essentially amounts to showing that the claims (i) and (ii) described in \cref{sec:techniques-single-flat} hold in each iteration. 

Let $\lambda_{\min} = \min_{q \in [r]} \lambda_q$ and $\lambda_{\max} = \max_{q \in [r]} \lambda_q$.
For some $0< \eps \leq 1/2$, let $\kappa \gtrsim \Paren{\frac{\eps}{1-\eps}}^{p-1}$ such that $\lambda_{\min} \geq \kappa \cdot \lambda_{\max}$.
Suppose that
\begin{equation}
	\label{eq:lambda-min-value}
	\lambda_{\min} \gtrsim \frac{1}{\eps \cdot (1-\eps)^p} \cdot \sqrt{t \Paren{\frac{k}{t}}^p \log n}
	\quad \text{and} \quad
	\lambda_{\min} \gtrsim \Paren{\frac{\eps}{1-\eps}}^{p-1} \cdot \lambda_{\max}\,.
\end{equation}

Consider an arbitrary iteration $i$ and suppose that we exactly recovered the support of one signal in each of the previous iterations.
Without loss of generality, assume that $\lambda_{\max}$ is the largest signal strength among the yet to be recovered signals, and let $x_{(\max)}$ be one such corresponding signal.

For $u \in U_t$ satisfying $\supp(u) \subseteq \supp(x_{(\max)})$, we have
\[
\iprod{\pmb Y^{(1)}, \tensorpower{u}{p}}
= \lambda_{\max} \iprod{x_{(\max)},u}^p + \iprod{\pmb W^{(1)}, \tensorpower{u}{p}}
\geq \lambda_{\max} \cdot \Paren{\frac{t}{k}}^{\frac{p}{2}}  - \cO\Paren{\sqrt{t\log n}}\,.
\]

On the other hand, for any $u \in U_t$ such that $\Card{\supp\Paren{u} \cap \supp\Paren{x_{(q)}}} \leq (1-\eps) \cdot t$ for all $q \in [r]$,
\begin{align*}
	\iprod{\pmb Y^{(1)}, \tensorpower{u}{p}}
	&=  \underset{q \in [r]}{\sum}\lambda_q \iprod{x_{(q)}, u}^p+ \iprod{\pmb W^{(1)}, \tensorpower{u}{p}}\\
	&\leq \lambda_{\max} \cdot \Paren{\frac{t}{k}}^{\frac{p}{2}} \cdot \Paren{\Paren{1-\eps}^p +\eps^p} + \cO\Paren{\sqrt{t\log n}}\\
	&\leq \lambda_{\max} \cdot \Paren{\frac{t}{k}}^{\frac{p}{2}} \cdot \Paren{1-\eps}^{p-1} + \cO\Paren{\sqrt{t\log n}} \,.
\end{align*}

Thus, as in \cref{sec:techniques-single-flat}, it follows that $\pmb v_*$ satisfies $\Card{\supp\Paren{\pmb v_*} \cap \supp\Paren{x_{(i)}}} \geq (1-\eps) \cdot t$ for some signal $x_{(i)}$.
Note that $x_{(i)}$ may not be $x_{(\max)}$.
Even though $\pmb v_*$ does not exactly overlap with any of the signal vectors, we will not accumulate an error at each iteration.
This is because, analogous to the single spike setting, we can exactly identify the support of a signal through $\pmb \alpha$.
For any $\ell \in \supp\Paren{x_{(i)}}$, it holds that $\pmb \alpha_\ell \gtrsim \sqrt{\log n}$ as before because $\Abs{\iprod{\pmb W^{(2)}, \tensorpower{\pmb v_*}{p-1} \otimes e_\ell}} \lesssim \sqrt{\log n}$.
Conversely, since signal supports are disjoint, we see that for $\ell \notin \supp\Paren{x_{(i)}}$,
\begin{align*}
	\pmb \alpha_\ell
	&=  \underset{q \in [r]}{\sum} \lambda_q \cdot x_{(q), \ell} \cdot \iprod{x_{(q)}, \pmb v_*}^{p-1} + \iprod{\pmb W^{(2)}, \tensorpower{\pmb v_*}{p-1} \otimes e_\ell}\\
	&\leq \lambda_{\max} \cdot \frac{\eps^{p-1}}{\sqrt{k}} \cdot \Paren{\frac{t}{k}}^{\frac{p-1}{2}}  + \cO\Paren{\sqrt{\log n}}\\
	& \leq \frac{\lambda_{\min}}{\kappa}\cdot \frac{\eps^{p-1}}{\sqrt{k}}\cdot\Paren{\frac{t}{k}}^{\frac{p-1}{2}}  + \cO\Paren{\sqrt{\log n}}\\
	&\lesssim \sqrt{\log n}\,.
\end{align*}
So, once again, $\pmb \alpha$ exactly identifies the support of $x_{(i)}$ with high probability.

\begin{myremark}[On the strength of the assumption on $\kappa$]
\textnormal{
As already briefly discussed in \cref{sec:results}, the algorithm provides a three-way trade-off between signal gap $\kappa$, signal-to-noise ratio $\lambda$ and running time.
By appropriately choosing the constant $\eps>0$, the algorithm can tolerate different values of $\kappa$.
Indeed, the above analysis holds as long as $\kappa \gtrsim \Paren{\frac{\eps}{1-\eps}}^{p-1}$.
This suggests two ways in which we can loosen the requirement $\lambda_{\min} \geq \kappa \cdot \lambda_{\max}$ and still successfully recover the signals through \cref{alg:multi-spike}.
One is increase the running time, so that we can decrease $\epsilon$ without increasing the signal-to-noise ratio $\lambda_{\min}$.
The other is to decrease $\epsilon$ and increase the value of $\lambda_{\min}$ accordingly.
 }
\end{myremark}

\begin{myremark}[On independent copies of $\pmb Y$]
\label{remark:independence}
\textnormal{
To clarify why it suffices to have 2 independent copies of $\pmb Y$ even for multiple iterations, observe that at each iteration \emph{i}, the choice of the set $\cI_i$  depends only on the vector $\pmb v_*$ with high probability.
Consider the following thought experiment where we are given a fresh copy $\pmb Y^{(i)}$ of $\pmb Y$ in the second phase of each iteration \textit{i} of the algorithm (while still using only a single copy $\pmb Y^{(1)}$for \emph{all} the first phases).
Even with fresh randomness, the result is the same as \cref{alg:multi-spike} with high probability because at each iteration the choice of maximizer $\pmb v_*$ causes the same output.
}
\end{myremark}

\begin{myremark}[Reconstructing the signals from their supports]
	\label{remark:reconstruct}
	\textnormal{
		After recovering individual signal supports, one can reconstruct signals using known tensor PCA algorithms (e.g.\ \cite{richard2014statistical,hopkins2015tensor}) on the subtensor defined by each recovered support.
		The signal strength required for this new subproblem is weaker and is satisfied by our recovery assumptions.
		For instance, by concatenating our algorithm with \cite[Theorem 7.1]{hopkins2015tensor}, one obtains vectors $\widehat{x}_{(1)}, \ldots, \widehat{x}_{(r)}$ such that $\Abs{\iprod{\widehat{x}_{(i)}, x_{(i)}}} \geq 0.99$, for any $i \in [r]$, with probability 0.99.
	}
\end{myremark}

\subsection{Approximately flat signals}
\label{sec:techniques-multi}

By factoring $A \geq 1$ into our assumptions on minimal signal strength $\lambda_{\min}$ and relative strength ratio $\kappa$, we can extend the above analyses (using the same proof outline as in \cref{sec:techniques-multi-flat}) so that \cref{alg:multi-spike} recovers the individual supports of multiple approximately flat $(k,A)$-sparse signals.
Besides accounting for $A$ factors, the only significant change in the analysis is in how we lower bound $\iprod{\pmb Y^{(1)}, \pmb v_*}$. 
Consider the first iteration (by our discussion in \cref{sec:techniques-multi-flat}, other iterations are similar) and let 
$u_*\in U_t$ be the vector satisfying
\begin{align*}
	\lambda \iprod{x, u_*}^p = \max_{q \in [r], u \in U_t} \lambda_q \iprod{x_{(q)}, u}^p\,.
\end{align*}
This choice allows us to account for skewed signals.
Let $\lambda_{\min} = \min_{q \in [r]} \lambda_q$ and $\lambda_{\max} = \max_{q \in [r]} \lambda_q$.
For some $0< \eps \leq 1/2$.
Suppose that
\begin{equation}
	\label{eq:lambda-min-apx-value}
	\lambda_{\min} \gtrsim \frac{A^{p}}{\eps \cdot (1-\eps)^p} \cdot \sqrt{t \Paren{\frac{k}{t}}^p \log n}
	\quad \text{and} \quad
	\lambda_{\min} \gtrsim A^{2p} \cdot \Paren{\frac{\eps}{1-\eps}}^{p-1} \cdot \lambda_{\max}\,.
\end{equation}

By definition of $\pmb v_*$, we have 
\begin{align}
	\label{eq:lower-bound-on-u-star}
	\iprod{\pmb Y^{(1)}, \pmb v_*} \geq \iprod{\pmb Y^{(1)}, u_*} = \lambda \iprod{x, u_*}^p + \iprod{\pmb W^{(1)}, \tensorpower{u_*}{p}}\,.
\end{align}
Conversely, consider an arbitrary $u \in U_t$ such that $\Card{\supp(u) \cap \supp\Paren{x_{(q)}}} < (1-\eps) \cdot t$ for all $q \in [r]$.
Intuitively, the largest attainable value for $\iprod{\pmb Y^{(1)}, \tensorpower{u}{p}}$ is obtained removing $\eps \cdot t$ entries from $u_*$ and placing them on some highly skewed signal.
Using \cref{eq:lambda-min-apx-value} and \cref{eq:lower-bound-on-u-star}, it is possible to show\footnote{For the full derivation, see the proof of \cref{lem:apxflat-maximizer} in \cref{sec:missing-recovery-proofs}. A $\sqrt{2}$ factor appears due to \cref{alg:preprocessing}.}
\begin{align*}
	\iprod{\pmb Y^{(1)}, \tensorpower{u}{p}}
	&\leq \lambda \iprod{x, u_*}^p \cdot \Paren{1 - \frac{\eps}{A^2}}^{p-1} + \iprod{\pmb W^{(1)}, \tensorpower{u}{p}}\,.
\end{align*}
Thus, it follows that $\pmb v_*$ satisfies $\Card{\supp\Paren{\pmb v_*} \cap \supp\Paren{x_{(i)}}} \geq (1-\eps) t$ for some $i \in [r]$.
We can now repeat the same analysis as \cref{sec:techniques-multi-flat} to argue that $\pmb \alpha$ behaves as an indicator vector for $\supp\Paren{x_{(i)}}$, taking account of $A$ factors.

\section{Preliminaries}
\label{sec:appendix_preliminaries}

\subsection{Packings and nets}

Packings and nets are useful in helping us discretize a possibly infinite metric space.
Let $\cX$ be a set of points and $d : \cX \times \cX \rightarrow \R^+$ be a (pseudo)metric\footnote{A metric satisfies 3 properties: (1) $d(x,y) = 0$ if and only if $x = y$; (2) $d(x,y) = d(y,x)$; (3) $d(x,y) \leq d(x,z) + d(z,y)$. A pseudometric may violate (1) by allowing $d(x,y) = 0$ for distinct $x \neq y$. Pseudometrics are also sometimes called semimetrics.} for $\eps > 0$.
That is, $(\cX, d)$ is a (pseudo)metric space.
An $\eps$-packing $\cX' \subseteq \cX$ is a subset where any two distinct points $x,y \in \cX'$ have distance $d(x,x') > \eps$.
An $\eps$-net\footnote{Nets are also referred to as coverings.} $\cX' \subseteq \cX$ is a subset such that for any point $x \in \cX$, there exists some point $x' \in \cX'$ (possibly itself) where $d(x,x') \leq \eps$.
Under these notions, the covering number $N(\cX, d, \eps)$ and packing number $P(\cX, d, \eps)$ are defined as the size of the \emph{smallest} $\eps$-net of $\cX$ and \emph{largest} $\eps$-packing of $\cX$ respectively.
It is known\footnote{e.g.\ See resources such as \cite{tao2014metric} and \cite[Section 4.2]{vershynin2018high}.} that
\[
P(\cX, d, 2\epsilon) \leq N(\cX, d, \epsilon) \leq P(\cX, d, \epsilon)
\]

\subsection{Hermite polynomials}
\label{sec:hermite-polys}

In this section, we introduce Hermite polynomials and state some properties used in our low-degree analysis.
For further details, see \cite[Section 11.2]{o2014analysis}.

\begin{definition}[Inner product of functions]
For a pair of functions $f$ and $g$ operating on the same domain $\mathcal{D}$, their inner product is defined by $\langle f, g \rangle = \E_{\pmb z \sim \mathcal{D}} [f(\pmb z) g(\pmb z)]$.
\end{definition}

The set $\{1, \pmb z, \pmb z^2, \ldots, \pmb z^D\}$ is a basis for the set of polynomials with maximum degree $D$ on Gaussian variable $\pmb z \sim N(\mu, 1)$.
By applying the Gram-Schmidt process and noting that odd functions in $\pmb z$ have expectation 0, we can diagonalize this set to obtain an orthogonal basis:
$H_{e_0}(\pmb z) = 1$, $H_{e_1}(\pmb z) = \pmb z$, $H_{e_2}(\pmb z) = \pmb z^2 - 1$, $H_{e_3}(\pmb z) = \pmb z^3 - 3 \pmb z$, etc.

The orthogonal basis for polynomials of maximum degree $D$ $\{H_{e_n}\}_{n \in [D]}$ is also called the \emph{probabilists' Hermite polynomials}.
It is known that
\[
\E_{\pmb z \sim N(\mu, 1)}[H_{e_n}(\pmb z)] = \mu^n
\quad
\text{ and }
\quad
\E_{\pmb z \sim N(\mu, 1)}[(H_{e_n}(\pmb z))^2] = n!
\]
As we are interested in \emph{orthonormal} bases, we use the \emph{normalized} probabilists' Hermite polynomials $\{h_n\}_{n \in \mathbb{N}}$ where $h_n = \frac{1}{\sqrt{n!}} H_{e_n}$.
One can check that
\[
\E_{\pmb z \sim N(\mu, 1)}[h_n(\pmb z)] = \frac{1}{\sqrt{n!}} \mu^n
\quad
\text{ and }
\quad
\E_{\pmb z \sim N(\mu, 1)}[(h_n(\pmb z))^2] = 1
\]

\subsection{Information theory}
\label{sec:info-theory}

Techniques from the statistical minimax theory, such as the Fano method, allow us to \emph{lower bound} the worst case behavior of \emph{any} estimator.
In the following discourse, we borrow some notation from \cite{duchi2016lecture}.
Given a single tensor observation $\pmb Y = \pmb W + \lambda \tensorpower{x}{p}$ generated from an underlying signal $x \in U_t$ (i.e.\ the parameter of the observation is $\theta(\pmb Y) = x$), an estimator $\widehat{\theta}(\pmb Y)$ outputs some unit vector in $\widehat{x} \in U_k$.
For two vectors $x$ and $x'$, we use the pseudometric\footnote{Instead of the ``standard'' $\Norm{x-x'}_2$ loss, we want a loss function that captures the ``symmetry'' that $\langle x, x' \rangle^p = \langle x, -x' \rangle^p$ for even tensor powers $p$. Clearly, $\rho(x,y) = \rho(y,x)$ and one can check that $\rho(x,y) \leq \rho(x,z) + \rho(z,y)$. Observe that $\rho$ is a pseudometric (and not a metric) because $\rho(x,y) = 0$ holds for $x = -y$.} $\rho(x,x') = \min \{ \Norm{x-x'}_2, \Norm{x+x'}_2 \}$ and the loss function $\Phi(t) = t^2 / 2$.
Thus, $\Phi(\rho(x,x')) = 1 - \Abs{\langle x, x' \rangle}$ with the corresponding minimax risk being
\[
\inf_{\widehat{\theta}} \sup_{x \in U_k} \E_{\pmb{Y}}\left[ \Phi \left( \rho \left( \widehat{\theta}(\pmb Y), \theta(\pmb Y) \right) \right) \right]
= \inf_{\widehat{x} \in U_k} \sup_{x \in U_k} \E_{\pmb Y} \left[ 1 - \Abs{\langle \widehat{x}, x \rangle} \right]
\]

A common way to \emph{lower bound} the minimax risk function is to look at it from the lens of a \emph{finite} testing problem.
The \emph{canonical hypothesis testing problem}\footnote{See Section 13.2.1 in \cite{duchi2016lecture}.} (in our context) is as follows.
Let $\cX \subseteq U_k$ be an $\eps$-packing of $U_k$ of size $\Abs{\cX} = m \geq P(U_k, \rho, \eps)$.
That is, $\min_{x_i, x_j \in \cX, i \neq j} \rho(x_i, x_j) > \eps$.
Then, (1) Nature chooses a unit vector $\pmb x \in \cX$ \emph{uniformly at random}; (2) We observe tensor $\pmb Y = \pmb W + \lambda \tensorpower{\pmb x}{p}$; (3) A test $\Psi : \cY \rightarrow \cX$ determines what is the planted unit vector.
Applying Fano's inequality (\cref{lem:fano}) and the data processing inequality (\cref{lem:data-processing}) to the above-mentioned canonical hypothesis testing problem, one can show\footnote{E.g.\ see \cite[Proposition 13.10]{duchi2016lecture}, adapted to our context. Recall that we had $\Phi(t) = t^2 / 2$.} that
\[
\inf_{\widehat{x} \in U_k} \sup_{x \in U_k} \E_{\pmb Y} \left[ 1 - \Abs{\langle \widehat{x}, x \rangle} \right]
\geq
\inf_{\widehat{x} \in U_k} \sup_{x \in \cX} \E_{\pmb Y} \left[ 1 - \Abs{\langle \widehat{x}, x \rangle} \right]
\geq 
\frac{\eps^2}{4} \cdot \left( 1 - \frac{I(x; \pmb Y) + 1}{\log m} \right)
\]
where $I(x; \pmb Y)$ is the mutual information between $x$ and $\pmb Y$.
Since it is known\footnote{E.g.\ see \cite[Lemma 4]{scarlett2019introductory}.} that for $x \in \mathcal{X}$, one can upper bound $I(x ; \pmb Y)$ by
$
I(x ; \pmb Y)
\leq \max_{u, v \in \cX} D_{KL} \left( \bbP_{\pmb Y \sim \cY \mid u} \Bigm\Vert \bbP_{\pmb Y \sim \cY \mid v} \right)
$, we have
\begin{equation}
\inf_{\widehat{x} \in U_k} \sup_{x \in U_k} \E_{\pmb Y} \left[ 1 - \Abs{\langle \widehat{x}, x \rangle} \right]
\geq 
\frac{\eps^2}{4} \cdot \left( 1 - \frac{\max_{u, v \in \cX} D_{KL} \left( \bbP_{\pmb Y \sim \cY \mid u} \Bigm\Vert \bbP_{\pmb Y \sim \cY \mid v} \right) + 1}{\log m} \right)
\end{equation}
where $D_{KL}(\cdot \Vert \cdot)$ is the KL-divergence function and $\bbP_{\pmb Y \sim \cY \mid u}$ is the probability distribution of observing $\pmb Y$ from signal $u$ with additive standard Gaussian noise tensor $\pmb W$.
\bigskip

We now state standard facts regarding Fano's inequality without proof.
For an introductory exposition on Fano's inequality and its applications, we refer readers to \cite{scarlett2019introductory}.

\begin{lem}[Fano's inequality (Uniform input distribution)]
\label{lem:fano}
Let $X, Y \in \mathcal{X}$ denote the (hidden) input and (observed) output.
Given $Y$, let $\widehat{X} \in \mathcal{X}$ be the estimated version of $X$ by \emph{any} estimator, and $P_e = \bbP[X \neq \widehat{X}]$ be the event that the estimation is wrong.
If $X$ is uniformly distributed over $\mathcal{X}$, then
\[
P_e \geq 1 - \frac{I(X;\widehat{X}) + 1}{\log \Abs{ \mathcal{X} }}
\]
where $I(X;\widehat{X}) = H(X) - H(X \mid \widehat{X})$ is the mutual information function.
\end{lem}

The following inequality makes the Fano's inequality more user-friendly since it replaces the $I(X;\widehat{X})$ term with $I(X;Y)$.
In statistical learning, it is typically easier to bound $I(X;Y)$ as we know how $Y$ is generated given $X$.

\begin{lem}[Data processing inequality]
\label{lem:data-processing}
Suppose variables $X$, $Y$ and $\widehat{X}$ form a Markov chain relation $X \rightarrow Y \rightarrow \widehat{X}$.
That is, $X$ and $\widehat{X}$ are independent given $Y$.
Then, $I(X ; Y) \geq I(X ; \widehat{X})$.
\end{lem}

\subsection{Low-degree method}
\label{sec:low-degree-method}

The low-degree likelihood ratio is a proxy to model efficiently computable functions.
It is closely related to the pseudo-calibration technique and it has been developed in a recent line of work on the Sum-of-Squares hierarchy (\cite{barak2019nearly,hopkins2017efficient,hopkins2017power,hopkins2018statistical}).
In this section, we will only introduce the basic idea and encourage interested readers to see \cite{hopkins2018statistical, conf/innovations/BandeiraKW20} for further details.

The objects of study are distinguishing versions of planted problems: given two distributions and an instance, the goal is to decide from which distribution the instance was sampled.
For us, the distinguishing formulation takes the form of deciding whether the tensor $\pmb Y$ was sampled according to the (planted) distribution as described in \cref{def:sstm}, or if it was sampled from the (null) distribution where  $\pmb W \in \otimes^p \R^n$ has  i.i.d.\ entries sampled from $N(0,1)$.
In general, we denote with $\nu$ the null distribution and with $\mu$ the planted distribution with the hidden structure.

\subsubsection{Background on Classical Decision Theory}
\label{sec:background-decision-theory}
From the point of view of classical Decision Theory, the optimal algorithm to distinguish between two distribution is well-understood. Given distributions $\nu$ and $\mu$ on a measurable space $\cS$, the likelihood ratio $L(\pmb Y):=d\bbP_\mu(\pmb Y)/d\bbP_\nu(\pmb Y)$\footnote{The Radon-Nikodym derivative.} is the optimal function to distinguish whether $\pmb Y\sim \nu$  or $\pmb Y\sim \mu$ in the following sense.
\begin{proposition}[\cite{neymanpearson}]
	If $\nu$ is absolutely continuous with respect to $\mu$, then the unique solution of the optimization problem
	\begin{align*}
	\max \E_\mu\Brac{f(\pmb Y)} \qquad \text{subject to }\E_\nu\Brac{f^2( \pmb Y)}=1
	\end{align*}
	is the normalized likelihood ratio $L(\pmb Y)/\E_\nu\Brac{L(\pmb Y)^2}$ and the optimum value is $\E_\nu\Brac{L(\pmb Y)^2}$.
\end{proposition}
Arguments about statistical distinguishability are also well-understood.
Unsurprisingly, the likelihood ratio plays a major role here as well and a key concept is Le Cam's contiguity. 
\begin{definition}[\cite{leCam}]
	Let $\underline{\mu}=\Paren{\mu_n}_{n\in \N}$ and $\underline{\nu}=\Paren{\nu_n}_{n\in \N}$ be sequences of probability measures on a common probability space $\cS_n$. Then $\underline{\mu}$ and $\underline{\nu}$ are \textit{contiguous}, written $\underline{\mu} \triangleleft \underline{\nu}$, if as $n\rightarrow \infty $, whenever for $A_n\in \cS_n$, $\bbP_{\underline{\mu}}(A_n)\rightarrow 0$ then $\bbP_{\underline{\nu}}(A_n)\rightarrow 0$.
\end{definition}
Contiguity allows us to capture the idea of indistinguishability of probability measures.
Two contiguous sequences $\underline{\mu},\underline{\nu}$ of probability measures are said to be indistinguishable if there is no function $f:\cS_n \rightarrow \Set{0,1}$ such that $f(\pmb Y)=1$ with high probability whenever $\pmb Y \sim \underline{\mu}$ and $f(\pmb Y)=0$ with high probability whenever $\pmb Y \sim \underline{\nu}$.
The \emph{second moment method} allows us to establish contiguity through the likelihood ratio.
\begin{proposition}
	If $\E_\nu\Brac{L_n(\pmb Y)^2}$ remains bounded as $n\rightarrow \infty$, then $\underline{\mu} \triangleleft\underline{\nu}$.
\end{proposition}
This discussion allows us to argue whether a given function can be used to distinguish between our planted and null distributions.
In particular, for probability measures $\mu$ and $\nu$ over $\cS$, and a given function $f:\cS\rightarrow \R$, we can say that $f$ cannot distinguish between $\mu$ and $\nu$ if it satisfies the following bound on the  $\chi^2$-divergence:
\begin{align*}
\frac{\Abs{\E_\mu(f(\pmb Y))-\E_\nu(f(\pmb Y))}}{\sqrt{\Var_\nu(f(\pmb Y))}}\leq o(1).
\end{align*}

\subsubsection{Background on the Low-degree Method}
\label{sec:background-low-degree-method}
The main problem with the likelihood ratio is that it is hard to compute in general, thus the analysis has to be restricted to the space of efficiently computable functions. Concretely, we use low-degree multivariate polynomials in the entries of the observation $\pmb Y$ as a proxy for efficiently computable functions.
By denoting the space of degree $\leq D$ polynomials in $\pmb Y$ with $\R_{\leq D}[\pmb Y]$, we can establish a low-degree version of the Neyman-Pearson lemma.
\begin{proposition}[e.g.\ \cite{hopkins2018statistical}]
	The unique solution of the optimization problem
	\begin{align*}
	\underset{f\in \R_{\leq D}[\pmb Y]}{\max} \E_\mu\Brac{f(\pmb Y)} \qquad \text{subject to }\E_\nu\Brac{F(\pmb Y)}=1
	\end{align*}
	is the normalized orthogonal projection $L^{\leq D}(\pmb Y)/\E_\nu\Brac{L^{\leq D}(\pmb Y)^2}$ of the likelihood ratio $L(\pmb Y)$ onto $\R_{\leq D}[\pmb Y]$ and the value of the optimization problem is $\E_\nu\Brac{L^{\leq D}(\pmb Y)^2}$.
\end{proposition}
With the reasoning above in mind, it is then natural to argue that a polynomial $p(\pmb Y)\in \R_{\leq D}[\pmb Y]$ cannot distinguish between $\mu$ and $\nu$ if 
\begin{align}\label{eq:low-degree-chi-squared}
\frac{\Abs{\E_\mu(p(\pmb Y))-\E_\nu(p(\pmb Y))}}{\sqrt{\Var_\nu(p(\pmb Y))}}\leq o(1).
\end{align}
It is important to remark that, at the heart of our discussion, there is the belief that in the study of planted problems, low-degree polynomials capture the computational power of efficiently computable functions. This can be phrased as the following conjecture.
\begin{conjecture}[Informal\protect\footnotemark]
\label{con:low-degree-polynomials}
\footnotetext{See \cite{barak2019nearly,hopkins2017efficient,hopkins2017power,hopkins2018statistical}.}
For ``nice'' sequences of probability measures $\underline{\mu}$ and $\underline{\nu}$, if there exists $D=D(d)\geq O\Paren{\log d}$ for which $\E_\nu\Brac{L^{\leq D}(\pmb Y)^2}$ remains bounded as $d\rightarrow \infty$, then there is no polynomial-time algorithm that distinguishes in the sense described in \cref{sec:background-decision-theory}\footnote{We do not explain what "nice" means (e.g.\ see \cite{hopkins2018statistical}) and remark that the most general formulation of the conjecture above (i.e.\ a broad definition of "nice" distributions) has been rejected (\cite{holmgren2020counterexamples}).}.
\end{conjecture}

A large body of work support this conjecture (see citations mentioned) by providing evidence of an intimate relation between polynomials, sum of squares algorithms, and lower bounds.

\subsubsection{Chi-squared divergence and orthogonal polynomials}
\label{sec:chi-square-divergence}
From a technical point of view, the key observation used to prove bounds for low-degree polynomials is the fact that the polynomial which maximizes the ratio in \cref{eq:low-degree-chi-squared} has a convenient characterization in terms of orthogonal polynomials with respect to the null distribution.

Formally, for any linear subspace of polynomials $\cS_{\le D} \subseteq \lowdegpolys{D}$ and any absolutely continuous probability distribution $\nu$ such that all polynomials of degree at most $2D$ are $\nu$-integrable, one can define an inner product in the space 
$\cS_{\le D}$ as follows
\[
\forall p, q \in \cS_{\le D}\quad \Iprod{p,q} = \E_{\pmb Y\sim\nu} p(\pmb Y)q(\pmb Y)\,.
\]
Hence we can talk about orthonormal basis in $\cS_{\le D}$ with respect to this inner product.
\begin{proposition}[See \cite{dOrsi2020} for a proof]
\label{proposition:optimal_polynomial_statistic}
Let $\cS_{\le D} \subseteq \lowdegpolys{D}$ be a linear subspace of polynomials of dimension $N$.
Suppose that $\nu$ and $\mu$ are probability distributions over $\pmb Y \in \R^{n\times d}$ such that any polynomial of degree at most $D$ is $\mu$-integrable and any polynomial of degree at most $2D$ is $\nu$-integrable.
Suppose also that $\nu$ is absolutely continuous. 
Let $\{\psi_i(\pmb Y)\}_{i=1}^{N}$ be an orthonormal basis in $\cS_{\le D}[\pmb Y]$  with respect to $\nu$. Then
\begin{equation*}
\underset{p \in\cS_{\le D}}{\max}
\frac{\Paren{\Ep p(\pmb Y)}^2}{\En p^2(\pmb Y)}=
\sum_{i=1}^{N}\Paren{\Ep\psi_i}^2.
\end{equation*}
\end{proposition}

In the case of Gaussian noise, a useful orthonormal basis in $\lowdegpolys{D}$ is the system of Hermite polynomials $\Set{H_\alpha(\pmb Y)}_{\Abs{\alpha} \leq D}$ (see \cref{sec:hermite-polys}).
By applying Proposition \ref{proposition:optimal_polynomial_statistic} to the subspace of polynomials such that $\En p(\pmb Y) = 0$, we get
\begin{corollary}
\label{corollary:optimal_polynomial_statistic_hermite}
Let $\nu$ be Gaussian.
Suppose that the distribution $\mu$ is so that any polynomial of degree at most $D$ is $\mu$-integrable. Then
\[
\underset{p \in \lowdegpolys{D}}{\max}\;
\frac{\Paren{\Ep p(\pmb Y)-\En p(\pmb Y)}^2}
{\Var_\nu p(\pmb Y)}\;
=
\underset{0<\card{\alpha}\le D}{\sum}\Paren{\Ep\hermitepoly{Y}{\alpha}}^2\,.
\]
\end{corollary}

\subsection{Sparse norm bounds}
\label{sec:sparse-norm-bounds}

Denote $\cS^{n-1} = \{ x \in \R^n : \Norm{x}_2 = 1 \}$ as the $n$-dimensional unit sphere and $A \in \R^{m \times n}$ be a matrix.
Then, the matrix norm of $A$ is defined as
\[
\Norm{A}
= \max_{x \in \cS^{n-1}} \Norm{Ax}_2
= \max_{x \in \cS^{m-1}, y \in \cS^{n-1}} x^\top A y
= \max_{x \in \cS^{m-1}, y \in \cS^{n-1}} \sum_{i=1}^m \sum_{j=1}^n A_{i,j} x_i y_j
\]
More generally, the tensor norm of an order $p \geq 2$ tensor $T \in \R^{n_1 \times \ldots \times n_p}$ is defined as
\begin{align*}
\Norm{T}
& = \max_{x_{(1)} \in \cS^{n_1 - 1}, \ldots, x_{(p)} \in \cS^{n_p - 1}} T \left( x_{(1)}, x_{(2)}, \ldots, x_{(p)} \right)\\
& = \max_{x_{(1)} \in \cS^{n_1 - 1}, \ldots, x_{(p)} \in \cS^{n_p - 1}} \sum_{i_1 = 1}^{n_1} \ldots \sum_{i_p = 1}^{n_p} T_{i_1, \ldots, i_p} x_{(1), i_1} x_{(2), i_2} \ldots x_{(p), i_p}
\end{align*}

In the following, let all dimensions be equal (i.e.\ $n = m = n_1 = \ldots = n_p$).
Without any sparsity conditions, it is known\footnote{e.g.\ See \cite[Section 4.4.2]{vershynin2018high}, \cite[Section 4.2.2]{tropp2015introduction} and \cite[Section 2.3.1]{tao2012topics}.} that $\Norm{\pmb A} \leq \cO \Paren{\sqrt{n} + t}$ with high probability for matrix $\pmb A$ with i.i.d.\ standard Gaussian entries.
For tensors, \cite{tomioka2014spectral} proved that $\Norm{\pmb T} \leq \cO \Paren{\sqrt{np \log(p) + \log\Paren{1/\gamma}}}$ with probability at least $1 - \gamma$.
These results are typically proven using $\eps$-net arguments over the unit sphere $\cS^{n-1}$.
However, to the best of our knowledge, there is no known result for bounding the norm of $\pmb A$ or $\pmb T$ when interacting with $r$ distinct (at most) $s$-sparse unit vectors from the set $\cS^{n-1}_s = \{ x \in \R^n : \Norm{x}_2 = 1,\ \Abs{\cI_x} \leq s \leq n \}$.

Using $\eps$-net arguments, we bound $\Abs{\pmb T \Paren{x_{(1)}, \ldots, x_{(p)}}}$ when $x_{(1)}, \ldots, x_{(p)}$ are $r$ distinct vectors from $\cS^{n-1}_s$.
Our result recovers known bounds, up to constant factors, when $s = n$ and $r = p$.

\begin{restatable}[Tensor sparse bound]{lem}{tensorsparsebound}
\label{lem:tensor-sparse-bound}
Let $\pmb T \in \otimes^p\R^n$ be an order $p \geq 2$ tensor with i.i.d.\ standard Gaussian entries and $x_{(1)}, \ldots, x_{(p)} \in \cS^{n - 1}_s$ be $r$ distinct (at most) $s$-sparse unit vectors.
Then, for $1 \leq s \leq n$, $1 \leq r \leq p$, and $\gamma \in (0,1)$,
\[
\bbP \left[ \Abs{\pmb T(x_{(1)}, \ldots, x_{(p)})} \geq \sqrt{8 \cdot \left( 4rs \ln \left( \frac{np}{s} \right) + \ln \left( \frac{1}{\gamma} \right) \right)} \right] \leq 2 \gamma
\]
\end{restatable}

\begin{proof}[Proof sketch of \cref{lem:tensor-sparse-bound} (described for $p=2$)]
We fix an $\eps$-net $\cN(\cS^{n-1}_s)$ of size $N(\cS^{n-1}_s, \eps)$ over sparse vectors, bound the norm for an arbitrary point in $\cN(\cS^{n-1}_s) \times \cN(\cS^{n-1}_s)$, and apply union bound over all points in $\cN(\cS^{n-1}_s) \times \cN(\cS^{n-1}_s)$.
Since $\cS^{n-1}_s \times \cS^{n-1}_s$ is compact, there is a maximizing point that attains the norm.
We complete the proof by relating the maximizer (and hence the norm) to its closest point in $\cN(\cS^{n-1}_s) \times \cN(\cS^{n-1}_s)$.
See \cref{sec:sparse-norm-bounds-proof} for the formal proof.
\end{proof}

\section{Limited Brute Force Recovery Algorithm}
\label{sec:appendix-limited-brute-force}

Following the discussions from \cref{sec:techniques-recovery}. the main goal of this section is to prove the most general form of our algorithmic result\footnote{\cref{thm:main-algorithm-informal} and \cref{thm:main-algorithm-multi-informal} are direct consequences of \cref{thm:apxflat}.} (\cref{thm:apxflat}) that our limited brute force search algorithm that recovers exact individual signal supports under some algorithmic assumptions.
Then, in \cref{sec:general-tensor-extension}, we explain how to extend our techniques to handle single-spike general tensors where the signal takes the form $x_{(1)} \otimes \ldots \otimes x_{(p)}$ involving $1 \leq \ell \leq p$ distinct $k$-sparse tensors.

\begin{restatable}[Multi-spike recovery for approximately flat signals]{mythm}{apxflatversion}
\label{thm:apxflat}
Consider \cref{def:sstm}.
Suppose
\[
\lambda_r \gtrsim \frac{\kappa}{(A \eps)^p} \sqrt{t \Paren{\frac{k}{t}}^p \ln \Paren{\frac{n}{\delta}}},
\quad
\lambda_r \geq \kappa \cdot \lambda_1,
\quad \text{and} \quad
\kappa \geq 5 A^{2p} \Paren{\frac{\eps}{1-\eps}}^{p-1}.
\]
Then, \cref{alg:multi-spike} that runs in $\cO(rpn^{p+t})$ time and, with probability at least $1-\delta$, outputs the individual signal supports $\supp\Paren{x_{(\pi(1))}}, \ldots, \supp\Paren{x_{(\pi(r))}}$ with respect to some unknown bijection $\pi: [r] \rightarrow [r]$.
\end{restatable}
\bigskip

We remark that the constant factors are chosen to make the analysis clean; smaller factors are possible.
As discussed in \cref{remark:reconstruct}, after recovering one can then run known tensor PCA recovery methods on the appropriate sub-tensor to obtain a good approximation of each signal.

For convenience, let us restate \cref{alg:preprocessing} and \cref{alg:multi-spike}.
\setcounter{algorithm}{0}
\begin{algorithm}[H]
	\caption{Preprocessing}
	\begin{algorithmic}
		\State \textbf{Input}: $\pmb Y$.
		\State Sample a Gaussian tensor $\pmb Z \in \otimes^p \R^n$ where each entry is an i.i.d.\ standard Gaussian $N(0,1)$.
		\State \textbf{Return} two independent copies $\pmb Y^{(1)}$ and $\pmb Y^{(2)}$ of $\pmb Y$ as follows:
		\[
		\pmb Y^{(1)} = \frac{1}{\sqrt{2}} \Paren{\pmb Y + \pmb Z}
		\quad \text{and} \quad
		\pmb Y^{(2)} = \frac{1}{\sqrt{2}} \Paren{\pmb Y - \pmb Z}
		\]
	\end{algorithmic}
\end{algorithm}

\setcounter{algorithm}{2}
\begin{algorithm}[H]
	\caption{Multi-spike limited brute force}
	\begin{algorithmic}
		\State \textbf{Input}: $k, t, r$ and $\pmb Y^{(1)}, \pmb Y^{(2)}$ obtained from \cref{alg:preprocessing}.
		\State \textbf{Repeat} for $i=1$ to $r$:
		\State $\qquad\quad $  Compute $\pmb v_*:=\argmax_{u \in U_t} \iprod{\pmb Y^{(1)}, \tensorpower{u}{p}}$ subject to $\supp\Paren{\pmb v_*} \cap \Paren{\underset{1 \leq j < i}{\bigcup} \cI_j} = \emptyset$.
		\State $\qquad\quad $ Compute the vector $\pmb \alpha\in \R^n$ with entries  $\pmb \alpha_\ell := \iprod{\pmb Y^{(2)}, \tensorpower{\pmb v_*}{p-1} \otimes e_\ell }$ for every $\ell \in [n]$.
		\State $\qquad\quad $  Let $\cI_i$ be the set of indices of the largest $k$ entries of $\pmb \alpha$.
		\State \textbf{Return} $\cI_{1},\ldots,\cI_r$.
	\end{algorithmic}
\end{algorithm}

Let us begin with a simple running time analysis.

\begin{restatable}[]{lem}{apxflatruntime}
\label{lem:apxflat-runtime}
\cref{alg:multi-spike} runs in $\cO(r p n^{p+t})$ time.
\end{restatable}
\begin{proof}[Proof of \cref{lem:apxflat-runtime}]
Sampling $\pmb Z$ and creating copies $\pmb Y^{(1)}$ and $\pmb Y^{(2)}$ take $\cO(n^p)$ time.
Fix an arbitrary round.
Observe that $\Abs{U_{t}} = \binom{n}{t} 2^{t} \leq \Paren{(2e)/t}^t n^t \leq e^2 n^t$.
Each computation of $\iprod{\pmb Y^{(1)}, \tensorpower{u}{p}}$ can be naively performed in $\cO(p n^p)$ time while checking whether for disjoint support can be done naively in additional $\cO(n^2)$ time for each $u \in U_t$.
Similarly, the computation of $\pmb \alpha$ can be done in $\cO(p n^p)$ time and we can perform a linear scan in $\cO(n)$ time to obtain the largest $k$ entries.
So, an arbitrary round runs in $\cO(p n^{p+t})$ time.
We perform the entire process $r$ times.
\end{proof}

\cref{alg:multi-spike} recovers the exact $k$-sparse support of some signal $x_{(\pi(i))}$ in each round.
That is, $\cI_{x_{(\pi(i))}} = \supp\Paren{{x_{(\pi(i))}}}$ for $i \in [r]$.
It succeeds when these two claims hold for any round $i \in [r]$:
\begin{enumerate}[(I)]
	\item The $t$-sparse maximizer $\pmb v_*$ shares $\geq (1-\eps) \cdot t$ support coordinates with some signal $x_{(\pi(i))}$.
	\item The $k$ largest entries of $\pmb \alpha$ belong to the support $\supp\Paren{x_{(\pi(i))}}$ of $x_{(\pi(i))}$.
\end{enumerate}
\cref{lem:apxflat-maximizer} and \cref{lem:apxflat-anchor} address these claims respectively.
See \cref{sec:missing-recovery-proofs} for their proofs.

\begin{restatable}[]{lem}{apxflatmaximizer}
\label{lem:apxflat-maximizer}
Consider \cref{def:sstm} and an arbitrary round $i \in [r]$.
Suppose
\[
\lambda_r \geq \frac{32 \kappa}{(A \eps)^p} \sqrt{t \Paren{\frac{k}{t}}^p \ln(n)},
\quad
\lambda_r \geq \kappa \cdot \lambda_1,
\quad \text{and} \quad
\kappa \geq 5 A^{2p} \Paren{\frac{\eps}{1-\eps}}^{p-1}.
\]
Then,
\[
\bbP \Brac{ \Abs{\supp\Paren{\pmb v_*} \cap \supp\Paren{x_{(\pi(i))}}} \geq (1 - \eps) \cdot t}
\geq 1 - 4 \exp \Paren{- \lambda_r^2 \frac{\eps^2}{128 A^4} \Paren{\frac{t}{k}}^p}
\]
\end{restatable}

\begin{restatable}[]{lem}{apxflatanchor}
\label{lem:apxflat-anchor}
Consider \cref{def:sstm} and an arbitrary round $i \in [r]$.
Suppose
\[
\lambda_r \geq \frac{32 \kappa}{(A \eps)^p} \sqrt{t \Paren{\frac{k}{t}}^p \ln(n)},
\quad
\lambda_r \geq \kappa \cdot \lambda_1,
\quad \text{and} \quad
\kappa \geq 5 A^{2p} \Paren{\frac{\eps}{1-\eps}}^{p-1}.
\]
Further suppose that $\Abs{\supp\Paren{\pmb v_*} \cap \supp\Paren{x_{(\pi(i))}}} \geq (1 - \eps) \cdot t$.
Then, the largest (in magnitude) coordinates of $\pmb \alpha$ are $\supp\Paren{x_{(\pi(i))}}$ with probability at least $1 - 2n \exp \Paren{- \lambda_r^2 \frac{A^{2p} \eps^{2p-2}}{16 \kappa^2 t} \Paren{\frac{t}{k}}^{p}}$.
\end{restatable}

We now prove \cref{thm:apxflat} using \cref{lem:apxflat-runtime}, \cref{lem:apxflat-maximizer}, and \cref{lem:apxflat-anchor}.

\begin{proof}[Proof of \cref{thm:apxflat}]
\cref{lem:apxflat-runtime} gives the running time.
The correctness of \cref{alg:multi-spike} hinges on \cref{lem:apxflat-maximizer} and \cref{lem:apxflat-anchor} always succeeding.
In each round, we need \cref{lem:apxflat-maximizer} to succeed once and \cref{lem:apxflat-anchor} to succeed at most $n$ times.
There are a total of at most $r(1+n)$ events, each failing with probability at most
$
4n \exp \Paren{- \lambda_r^2 \frac{A^{2p-4} \eps^{2p-2}}{128 \kappa^2 t} \Paren{\frac{t}{k}}^{p}}
$.
By union bound, the probability of \emph{any} event fails is at most
$
r(1+n) \cdot 4n \exp \Paren{- \lambda_r^2 \frac{A^{2p-4} \eps^{2p-2}}{128 \kappa^2 t} \Paren{\frac{t}{k}}^{p}}
$.
By the disjoint signal support assumption, we have $r \leq n$.
So, when
$
\lambda_r \gtrsim \frac{\kappa}{(A \eps)^p} \sqrt{t \Paren{\frac{k}{t}}^p \ln \Paren{\frac{n}{\delta}}},
$
\cref{alg:multi-spike} succeeds with probability at least $1-\delta$.
\end{proof}

\subsection{General tensors for single spike}
\label{sec:general-tensor-extension}

We now briefly describe how to extend the model of \cref{def:sstm} to the case where the single tensor signal could be made up of $1 \leq \ell \leq p$ distinct $k$-sparse vectors\footnote{This model has been studied by \cite{DBLP:journals/corr/abs-2005-10743}. To be precise, they actually allow different \emph{known} sparsity levels for each $x_{(q)}$ vector. Here, we assume that all of them are $k$-sparse. It is conceptually straightforward (but complicated and obfuscates the key idea) to extend the current discussion to allow different sparsity values.}: instead of the signal being $\tensorpower{x}{p}$, it is $x_{(1)} \otimes \ldots \otimes x_{(p)}$ involving $\ell$ distinct vectors.
The discussions in this section can be further generalized to the case of multiple approximately flat spikes using the techniques from \cref{sec:techniques-multi-flat} and \cref{sec:techniques-multi}.

Given $\ell$, we can modify \cref{alg:basic} to search over $U_t^{\otimes \ell}$ and compute $\pmb v_*$ that maximizes $\iprod{\pmb Y^{(1)}, u_{(1)} \otimes \ldots \otimes u_{(p)}}$ where there are $\binom{p-1}{\ell-1}$ possible ways\footnote{There are $\binom{p-1}{\ell-1}$ ways to obtain integer solutions to $x_1 + \ldots + x_\ell = p$ assuming $x_1 \geq 1, \ldots, x_\ell \geq 1$.} to form the signal using $\ell$ distinct $t$-sparse vectors.
By \cref{lem:tensor-sparse-bound}, $\iprod{\pmb W^{(1)}, u_{(1)} \otimes \ldots \otimes u_{(p)}} \lesssim \sqrt{\ell t \log (n)}$ whenever one (or more) of the $t$-sparse vectors used to form $u_{(1)} \otimes \ldots \otimes u_{(p)}$ is \emph{not} part of the actual signal.
Suppose the maximizer $\pmb v_*$ involves $\ell$ distinct vectors $\pmb v_{*,(1)}, \ldots, \pmb v_{*,(\ell)}$.
For notational convenience, let us write $\pmb v_* \setminus \pmb v_{*,(q)}$ to mean the tensor of order $p-1$ derived by removing one copy of $\pmb v_{*,(q)}$ from $\pmb v_*$.
Then, for each distinct vector $\pmb v_{*,(q)}$ in the maximizer, define $\pmb \alpha_{(q)} \in \mathbb{R}$ where $\alpha_{(q),i} = \langle \pmb Y^{(2)}, (\pmb v_* \setminus \pmb v_{*,(q)}) \otimes e_i \rangle$ and output the $k$ largest entries of $\pmb \alpha_{(q)}$ as the support of $\pmb v_{*,(q)}$.
This modified algorithm will run in time $\cO\Paren{\ell (pe)^\ell n^{p + \ell t}}$.\footnote{The $\ell$ factor is due to using $\ell$ copies of $\pmb \alpha$. The increase from $p$ to $p^\ell$ is due to trying $\binom{p-1}{\ell-1}$ combinations. The increase of $n^t$ to $n^{\ell} e^{\ell}$ is due to searching over $U_t^{\otimes \ell}$.}
Adapting our analysis for the single spike accordingly (by using \cref{lem:tensor-sparse-bound} with $\ell$ distinct vectors) will show that we can recover the signal supports of each $u_{(q)}$ whenever $\lambda \gtrsim \sqrt{\ell t (k/t)^p \log n}$.\footnote{The extra $\sqrt{\ell}$ factor follows from \cref{lem:tensor-sparse-bound} to accomodate $\ell$ distinct vectors in the maximization.}
Notice that this is an improvement over the algorithm of \cite{DBLP:journals/corr/abs-2005-10743} for $\ell \in o(p)$ or $t \geq 2$ when $p \in \omega(1)$.

\section{Computational Low-Degree Bounds}
\label{sec:appendix-computational-bounds}

In this section, we formalize our results on the computational low-degree bounds for sparse tensor PCA.
We will first show a low-degree lower bound on the distinguishing problem using $k$-sparse scaled Rademacher unit vectors and then will give a low-degree distinguishing algorithm showing that the lower-bound is tight in certain parameter regimes.

Following the discussion in \cref{sec:low-degree-method}, we design the following distinguishing problem.

\begin{problem}[Hypothesis testing for single-spiked $k$-sparse scaled Rademacher vectors]
\label{problem:hypo-testing}
Given an observation tensor $\pmb Y \in \otimes^p\R^n$, decide whether:
\begin{align*}
\text{Null distribution $H_0$}: &\; \pmb Y = \pmb W\\
\text{Planted distribution $H_1$}: &\; \pmb Y = \pmb W + \lambda \tensorpower{\pmb x}{p}
\end{align*}
where $\pmb W$ is a noise tensor with i.i.d.\ $N(0,1)$ entries and $\pmb x$ is a $k$-sparse scaled Rademacher unit vector whose entries are independently drawn as follows:
\begin{align*}
	\pmb x_i = \begin{cases}
	1/\sqrt{k} &\text{ with probability }k/(2n),\\
	-1/\sqrt{k} &\text{ with probability }k/(2n),\\
	0 &\text{ with probability }1-k/n\,.
	\end{cases}
\end{align*}
\end{problem}

Formally speaking, the vector $\pmb x$ in \cref{problem:hypo-testing} is not necessarily a unit vector, as compared to the planted signal in the single-spike case of \cref{def:sstm}.
However, since $\pmb x$ is $k(1+o(1))$-sparse with high probability, a lower bound given by \cref{problem:hypo-testing} implies a distinguishing lower bound for single-spike sparse tensor model with a planted $k(1+o(1))$-sparse vector and signal strength $\frac{\lambda}{1+o(1)}$.
We study \cref{problem:hypo-testing} through the lens of low-degree polynomials.
Since $\pmb W$ is Gaussian noise, we use the set of normalized probabilists' Hermite polynomials $\{h_\alpha\}_{\alpha}$ as our orthogonal basis.
Our strategy is similar to prior works such as \cite{hopkins2017power,hopkins2017efficient,ding2019subexponential}: By examining the low-degree analogue of the $\chi^2$-divergence between probability measures, we will show that low-degree polynomial estimators cannot distinguish $H_0$ and $H_1$.

We now state the two main theorems that we will prove in the following subsections.
For a cleaner exposition, we defer some proofs to \cref{sec:comp-bound-proofs}.

\begin{restatable}[Single-spike low-degree distinguishability lower bound]{mythm}{lowdeglowerbound}
\label{thm:low-deg-lower-bound}
Let $p \geq 2$, $1 \leq D \leq 2n/p$, $\pmb Y \in \otimes^p\R^n$ be an observation tensor, $\pmb x$ be a $k$-sparse scaled Rademacher vector, and $\{h_\alpha\}_{\alpha}$ be the set of normalized probabilists' Hermite polynomials.
If $0 \leq \eps \leq 1/2$ and
\[
\lambda \leq \sqrt{\frac{\eps D}{e 4^p}} \min \left\{ \left( \frac{n}{pD} \right)^{p/4},\; \left( \frac{k}{pD} \left( 1 + \Abs{\ln \left( \frac{npD}{ek^2} \right)} \right) \right)^{p/2} \right\},
\]
then
\[
\chi^2(H_1\ \Vert\ H_0)
= \sup_{|\alpha| \leq D} \frac{\left( \mathbb{E}_{H_1} h_{\alpha}(\pmb Y) - \mathbb{E}_{H_0} h_{\alpha}(\pmb Y) \right)^2}{Var_{H_0} h_{\alpha}(\pmb Y)}
= \sum_{|\alpha| \leq D} \left( \mathbb{E}_{H_1} h_{\alpha}(\pmb Y) \right)^2
\leq 2 \eps.
\]
\end{restatable}

\begin{restatable}[Single-spike low-degree distinguishability lower bound]{mythm}{lowdegupperbound}
\label{thm:low-deg-upper-bound}
Let $p \geq 2$, $1 \leq D \leq 2n/p$, $\pmb Y \in \otimes^p\R^n$ be an observation tensor, $\pmb x$ be a $k$-sparse scaled Rademacher vector, and $\{h_\alpha\}_{\alpha}$ be the set of normalized probabilists' Hermite polynomials.
If either of the following holds:
\begin{enumerate}
\item If $D$ is even and
\[
\lambda \geq \eps^{\frac{1}{2D}} e^{\frac{p}{2}} \sqrt{D} \left( \frac{n}{pD} \right)^{\frac{p}{4}}
\]
\item If $p \leq n$, $D \leq \frac{\ln^2 (n/p)}{4e^2}$ is even, $\sqrt{np} \cdot \left( \frac{e\sqrt{D}}{\ln(n/k)} \right) \leq k \leq \sqrt{np}$, and
\[
\lambda \geq \eps^{\frac{1}{2D}} \sqrt{D} \left( \frac{k}{pD} \ln \left( \frac{n}{k} \right) \right)^{\frac{p}{2}}
\]
\end{enumerate}
then
\[
\chi^2(H_1\ \Vert\ H_0)
= \sup_{|\alpha| \leq D} \frac{\left( \mathbb{E}_{H_1} h_{\alpha}(\pmb Y) - \mathbb{E}_{H_0} h_{\alpha}(\pmb Y) \right)^2}{Var_{H_0} h_{\alpha}(\pmb Y)}
= \sum_{|\alpha| \leq D} \left( \mathbb{E}_{H_1} h_{\alpha}(\pmb Y) \right)^2
\geq \eps.
\]
\end{restatable}

\subsection{Low-degree lower bound}

To prove our computational lower bound, we first compute $\left( \E_{H_1} h_{\alpha}(\pmb Y) \right)^2$ explicitly in \cref{lem:hermite-expectation} using properties of the normalized probabilists' Hermite polynomials for a given degree parameter $D$.
Then, we upper bound $\sum_{\Abs{\alpha} \leq D} \left( \E_{H_1} [h_{\alpha}(\pmb Y)] \right)^2$ using \cref{lem:hermite-expectation-upperbound} and \cref{lem:sum-upperbound}.
Solving for the condition on $\lambda$ such that $\sum_{\alpha} \left( \E_{H_1} [h_{\alpha}(\pmb Y)] \right)^2 \ll \eps$ yields our computational lower bound \cref{thm:low-deg-lower-bound}.

\begin{restatable}{lem}{hermiteexpectation}
\label{lem:hermite-expectation}
Let $p \geq 2$, $d \geq 1$, $\pmb Y \in \otimes^p\R^n$ be an observation tensor, $\pmb x$ be a $k$-sparse scaled Rademacher vector, and $\{h_\alpha\}_{\alpha}$ be the set of normalized probabilists' Hermite polynomials.
An entry of $\pmb Y \in \otimes^p\R^n$ can be indexed by either an integer from $[n^p]$ or a $p$-tuple.
Define $\phi: [n^p] \rightarrow [n]^p$, $\alpha$, $c(\alpha)$, $s(\alpha)$, and $\mathbbm{1}_{even(c(\alpha))}$ as follows:
\begin{itemize}
	\item $\phi(i)$ maps to a $p$-tuple indicating the $p$ (possibly repeated) entries of $\pmb x$ that are used.
	\item $\alpha = (\alpha_1, \ldots, \alpha_{n^p})$ is an $n^p$-tuple that corresponds to a Hermite polynomial of degree $|\alpha| = \sum_{i=1}^{n^p} \alpha_i$.
		For each $i$, $\alpha_i$ is the number of times entry $\pmb Y_{\phi(i)}$ was chosen, where each $\pmb Y_{\phi(i)}$ references $p$ coordinates of $\pmb x$.
	\item $c(\alpha) = (c_1, \ldots, c_n)$, where $c_j$ is the number of times $\pmb x_j$ is used in $\alpha$.
	\item $s(\alpha)$ is the number of distinct non-zero $\pmb x_j$'s in $c(\alpha)$.
	\item $\mathbbm{1}_{even(c(\alpha))}$ be the indicator whether \emph{all} $c_j$'s are even.
\end{itemize}
Under these definitions, we have the following:
\[
\left( \E_{H_1} h_{\alpha}(\pmb Y) \right)^2
= \lambda^{2d} k^{-pd} \mathbbm{1}_{even(c(\alpha))} \left( \frac{k}{n} \right)^{2 s(\alpha)} \left( \Pi_{i=1}^{n^p} \frac{1}{(\alpha_i)!} \right).
\]
\end{restatable}

\begin{restatable}{lem}{hermiteexpectationupperbound}
\label{lem:hermite-expectation-upperbound}
Let $p \geq 2$, $1 \leq D \leq 2n/p$, $\pmb Y \in \otimes^p\R^n$ be an observation tensor, $\pmb x$ be a $k$-sparse scaled Rademacher vector, and $\{h_\alpha\}_{\alpha}$ be the set of normalized probabilists' Hermite polynomials.
Then,
\[
\sum_{\Abs{\alpha} \leq D} \left( \E_{H_1} [h_{\alpha}(\pmb Y)] \right)^2
\leq \sum_{d=1}^D \frac{\lambda^{2d}}{d!} \sum_{s=1}^{pd/2} \left( \frac{ek^2}{sn} \right)^s \left( \frac{s}{k} \right)^{pd}.
\]
\end{restatable}

\begin{restatable}[]{lem}{sumupperbound}
\label{lem:sum-upperbound}
For $p \geq 2$, $d \geq 1$, $1 \leq k \leq n$ and $1 \leq s \leq pd/2$, we have
\[
\left( \frac{ek^2}{sn} \right)^s \left( \frac{s}{k} \right)^{pd}
\leq \left[ \frac{2pd}{\min \left\{ \sqrt{npd},\; k \left( 1 + \Abs{\ln \left( \frac{npd}{ek^2} \right)} \right) \right\}} \right]^{pd}.
\]
\end{restatable}

We are now ready to prove \cref{thm:low-deg-lower-bound}.

%\lowdeglowerbound*
\begin{proof}[Proof of \cref{thm:low-deg-lower-bound}]
\cref{lem:hermite-expectation-upperbound} and \cref{lem:sum-upperbound} together tell us that
\begin{align*}
\sum_{\Abs{\alpha} \leq D} \left( \E_{H_1} [f_{\alpha}(\pmb Y)] \right)^2
& \leq \sum_{d=1}^D \frac{\lambda^{2d}}{d!} \frac{pd}{2} \left[ \frac{2pd}{\min \left\{ \sqrt{npd},\; k \left( 1 + \Abs{\ln \left( \frac{npd}{ek^2} \right)} \right) \right\}} \right]^{pd}\\
& = \sum_{d=1}^D \lambda^{2d} \left[ \left( \frac{1}{2d!} \right)^{\frac{1}{pd}} \left( pd \right)^{\frac{1}{pd}} \frac{2pd}{\min \left\{ \sqrt{npd},\; k \left( 1 + \Abs{\ln \left( \frac{npd}{ek^2} \right)} \right) \right\}} \right]^{pd}\\
& \leq \sum_{d=1}^D \lambda^{2d} \left[ \left( \frac{e}{d} \right)^{\frac{1}{p}} \frac{4pd}{\min \left\{ \sqrt{npd},\; k \left( 1 + \Abs{\ln \left( \frac{npd}{ek^2} \right)} \right) \right\}} \right]^{pd} && \text{$(\star)$}\\
& = \sum_{d=1}^D \lambda^{2d} \left( \frac{e}{d} \right)^{d} \frac{4^{pd}}{\min \left\{ \left( \frac{n}{pd} \right)^{pd/2},\; \left( \frac{k}{pd} \left( 1 + \Abs{\ln \left( \frac{npd}{ek^2} \right)} \right) \right)^{pd} \right\}}\\
& \leq \sum_{d=1}^D \eps^d \left( \frac{D}{d} \right)^{d} \frac{\min \left\{ \left( \frac{n}{pD} \right)^{pd/2},\; \left( \frac{k}{pD} \left( 1 + \Abs{\ln \left( \frac{npD}{ek^2} \right)} \right) \right)^{pd} \right\}}{\min \left\{ \left( \frac{n}{pd} \right)^{pd/2},\; \left( \frac{k}{pd} \left( 1 + \Abs{\ln \left( \frac{npd}{ek^2} \right)} \right) \right)^{pd} \right\}} && \text{$(\ast)$}\\
& \leq \sum_{d=1}^D \eps^d && \text{$(\dag)$}
\end{align*}
where $(\star)$ is because $\frac{1}{2d!} \leq \frac{1}{d!} \leq \left( \frac{e}{d} \right)^d$ and $(pd)^{\frac{1}{pd}} \leq 2$, $(\ast)$ is the theorem assumption on $\lambda$, and $(\dag)$ is because $p \geq 2$.
The statement follows since $\sum_{d=1}^D \eps^d \leq \frac{\eps}{1-\eps} \leq 2 \eps$ for $0 \leq \eps \leq 1/2$.
\end{proof}

\subsection{Low-degree distinguishing algorithm}
\label{sec:low-deg-distinguishing-algo}

The starting point of our distinguishing algorithm is the explicit expression from \cref{lem:hermite-expectation} and \cref{claim:counting}.
Assuming $D$ is even\footnote{For $D \geq 2$, we consider Hermite polynomials of degree either $D$ or $D-1$ (whichever is even).}, we show that degree $D$ Hermite polynomials is ``sufficiently large'' by considering a subset of terms in the explicit summation.

\begin{proof}[Proof of \cref{thm:low-deg-upper-bound}]
Under the assumption of $D \leq \frac{2n}{p}$ and $D$ is even, \cref{lem:hermite-expectation} and \cref{claim:counting} together tell us that
\begin{align*}
\sum_{\Abs{\alpha} \leq D} \left( \E_{H_1} [f_{\alpha}(\pmb Y)] \right)^2
& = \sum_{d=1}^D \lambda^{2d} k^{-pd} \frac{1}{d!} \sum_{s=1}^{\left\lfloor pd/2 \right\rfloor} \binom{n}{s} \left( \frac{k}{n} \right)^{2s} \sum_{\substack{\beta_1 + \ldots + \beta_s = pd/2\\ \beta_1 \neq 0, \ldots, \beta_s \neq 0}} \binom{pd}{2 \beta_1, \ldots, 2 \beta_s}\\
& \geq \frac{\lambda^{2D}}{D!} k^{-pD} \sum_{s=1}^{pD/2} \binom{n}{s} \left( \frac{k}{n} \right)^{2s} \sum_{\substack{\beta_1 + \ldots + \beta_s = pD/2\\ \beta_1 \neq 0, \ldots, \beta_s \neq 0}} \binom{pD}{2 \beta_1, \ldots, 2 \beta_s}\\
& \geq \frac{\lambda^{2D}}{D!} k^{-pD} \binom{n}{pD/2} \left( \frac{k}{n} \right)^{pD} \binom{pD}{2, \ldots, 2} && \text{($\dag$)}\\
& \geq \frac{\lambda^{2D}}{D!} k^{-pD} \left( \frac{2n}{pD} \right)^{pD/2} \left( \frac{k}{n} \right)^{pD} \left( \frac{pD}{e} \right)^{pD} 2^{-pD/2} && \text{($\ddag$)}\\
& \geq \left( \frac{\lambda^{2}}{D} \left( \frac{kpD}{ek \sqrt{npD}} \right)^{p} \right)^D && \text{($\star$)}
\end{align*}
where ($\dag$) is by only using $s = pD/2$ and $\beta_1 = \ldots = \beta_{pD/2} = 1$, ($\ddag$) is due to $\binom{n}{k} \geq (n/k)^k$ and $n! \geq (n/e)^n$, and  ($\star$) is because $D! \leq D^D$.

When $\lambda \geq \eps^{\frac{1}{2D}} e^{\frac{p}{2}} \sqrt{D} \left( \frac{n}{pD} \right)^{\frac{p}{4}}$, we see that
\[
\chi^2(H_1\ \Vert\ H_0)
= \sum_{\Abs{\alpha} \leq D} \left( \E_{H_1} [f_{\alpha}(\pmb Y)] \right)^2
\geq \left( \frac{\lambda^{2}}{D} \left( \frac{kpD}{ek \sqrt{npD}} \right)^{p} \right)^D
= \left( \frac{\lambda^{2}}{De^p} \left( \frac{pD}{n} \right)^{\frac{p}{2}} \right)^D
\geq \eps
\]

We now assume that $p \leq n$, $D \leq \frac{\ln^2 (n/p)}{4e^2}$ and $\sqrt{np} \cdot \left( \frac{e\sqrt{D}}{\ln(n/k)} \right) \leq k \leq \sqrt{np}$.
Then,
\[
\left( \frac{\lambda^{2}}{D} \left( \frac{kpD}{ek \sqrt{npD}} \right)^{p} \right)^D
= \left( \frac{\lambda^{2}}{D} \left( \frac{pD}{k \ln (\frac{n}{k})} \right)^{p} \left( \frac{k \ln (\frac{n}{k})}{\sqrt{e^2 npD}} \right)^{p} \right)^D
\geq \left( \frac{\lambda^{2}}{D} \left( \frac{pD}{k \ln (\frac{n}{k})} \right)^{p} \right)^D
\]
where the last inequality is because $\sqrt{np} \cdot \left( \frac{e\sqrt{D}}{\ln(n/k)} \right) \leq k$.
The constraints $p \leq n$ and $\sqrt{D} e \leq \frac{1}{2} \ln (n/p) = \ln (n/\sqrt{np}) \leq \ln (n/k)$ ensure that there exists valid values of $k$.

So when $\lambda \geq \eps^{\frac{1}{2D}} \sqrt{D} \left( \frac{k}{pD} \ln \left( \frac{n}{k} \right) \right)^{\frac{p}{2}}$, we see that
\[
\chi^2(H_1\ \Vert\ H_0)
= \sum_{\Abs{\alpha} \leq D} \left( \E_{H_1} [f_{\alpha}(\pmb Y)] \right)^2
\geq \left( \frac{\lambda^{2}}{D} \left( \frac{pD}{k \ln (\frac{n}{k})} \right)^{p} \right)^D
\geq \eps
\]
\end{proof}

\section{Information-theoretic Lower Bound}
\label{sec:appendix-info-lower-bound}

In this section, we will use standard techniques\footnote{See \cref{sec:info-theory} for a brief introduction.} in information theory to lower bound the minimax risk for approximate signal recovery in the single-spike sparse tensor PCA.
Equivalent results appeared in \cite{DBLP:journals/corr/PerryWB16, DBLP:conf/nips/Niles-WeedZ20}.
Nevertheless, we include this section for completeness.

Consider the following setting: Given a single tensor observation $\pmb Y = \pmb W + \lambda \tensorpower{x}{p}$ generated from an underlying signal $x \in U_t$ (i.e.\ the parameter of the observation is $\theta(\pmb Y) = x$), an estimator $\widehat{\theta}(\pmb Y)$ outputs some unit vector in $\widehat{x} \in U_k$.
For two vectors $x$ and $x'$, we use the pseudometric\footnote{Instead of the ``standard'' $\Norm{x-x'}_2$ loss, we want a loss function that captures the ``symmetry'' that $\langle x, x' \rangle^p = \langle x, -x' \rangle^p$ for even tensor powers $p$. Clearly, $\rho(x,y) = \rho(y,x)$ and one can check that $\rho(x,y) \leq \rho(x,z) + \rho(z,y)$. Observe that $\rho$ is a pseudometric (and not a metric) because $\rho(x,y) = 0$ holds for $x = -y$.} $\rho(x,x') = \min \{ \Norm{x-x'}_2, \Norm{x+x'}_2 \}$ and the loss function $\Phi(t) = t^2 / 2$.
Thus, $\Phi(\rho(x,x')) = 1 - \Abs{\langle x, x' \rangle}$ with the corresponding minimax risk being
\[
\inf_{\widehat{\theta}} \sup_{x \in U_k} \E_{\pmb{Y}}\left[ \Phi \left( \rho \left( \widehat{\theta}(\pmb Y), \theta(\pmb Y) \right) \right) \right]
= \inf_{\widehat{x} \in U_k} \sup_{x \in U_k} \E_{\pmb Y} \left[ 1 - \Abs{\langle \widehat{x}, x \rangle} \right]
\]
Let $\cX \subseteq U_k$ be an $\eps$-packing of $U_k$ of size $\Abs{\cX} = m \geq P(U_k, \rho, \eps)$.
Then, one can show that the minimax risk can be lower bounded as follows:
\begin{equation}
\label{eqn:info-lb-eqn}
\inf_{\widehat{x} \in U_k} \sup_{x \in U_k} \E_{\pmb Y} \left[ 1 - \Abs{\langle \widehat{x}, x \rangle} \right]
\geq 
\frac{\eps^2}{4} \cdot \left( 1 - \frac{\max_{u, v \in \cX} D_{KL} \left( \bbP_{\pmb Y \sim \cY \mid u} \Bigm\Vert \bbP_{\pmb Y \sim \cY \mid v} \right) + 1}{\log m} \right)
\end{equation}
where $D_{KL}(\cdot \Vert \cdot)$ is the KL-divergence function and $\bbP_{\pmb Y \sim \cY \mid u}$ is the probability distribution of observing $\pmb Y$ from signal $u$ with additive standard Gaussian noise tensor $\pmb W$.
The following information-theoretic lower bound is shown by lower bounding \cref{eqn:info-lb-eqn}.

\begin{restatable}[Single-spike info-theoretic lower bound]{mythm}{minimaxbound}
\label{thm:minimax-bound}
Given $\pmb Y = \pmb W + \lambda \tensorpower{x}{p} \in \otimes^p\R^n$ where $\pmb W$ is a noise tensor with i.i.d.\ $N(0,1)$ entries and the planted signal $x \in U_k$ has signal strength $\lambda$.
Let $\widehat{x} \in U_k$ be the recovered signal by \emph{any} estimator.
If $n \geq 2k$ and $\lambda \leq \sqrt{\frac{k}{12} \log \left( \frac{n-k}{k} \right) - \frac{1}{2}}$, then
\[
\inf_{\widehat{x} \in U_k} \sup_{x \in U_k} \E_{\pmb Y} \left[ 1 - \Abs{\langle \widehat{x}, x \rangle} \right] \geq 0.05.
\]
\end{restatable}

\paragraph{Remark}
With $n \geq 2k$, we see that $\log (\frac{n-k}{k}) \geq \log (\frac{n}{2k})$.
Then, for if $n \geq 2^{\frac{1}{1-c}} k$ for any constant $c \in (0,1)$, we see that $\log (\frac{n}{2k}) \geq c \log (\frac{n}{k})$.
In particular, when $n \geq 4k$, we have $\log (\frac{n}{2k}) \geq \frac{1}{2} \log (\frac{n}{k})$ and can write $\lambda \lesssim \sqrt{k \log (n/k)}$.

To prove the result, we lower bound $m$ and upper bound the KL-divergence.
Since $\Norm{x-x'}_2 \geq \rho(x,x')$, we see that $N(U_k, \Norm{\cdot}_2, \eps) \leq P(U_k, \Norm{\cdot}_2, \eps) \leq P(U_k, \rho, \eps) \leq m$.
To lower bound $m$, we lower bound $N(U_k, \Norm{\cdot}_2, \eps)$ via \cref{lem:rademacher-eps-net-lower-bound}.
Then, we upper bound the KL-divergence in \cref{lem:KL-upperbound} by the triangle inequality and the KL-divergence of Gaussian vectors.
We defer the proofs of \cref{lem:rademacher-eps-net-lower-bound} and \cref{lem:KL-upperbound} to \cref{sec:info-lower-bound-proofs}.

\begin{restatable}[]{lem}{rademacherepsnetlowerbound}
\label{lem:rademacher-eps-net-lower-bound}
Let $U_k$ be the set of $k$-sparse flat unit vectors and $N(U_k, \Norm{\cdot}_2, \eps)$ be the $\eps$-covering number of $U_k$ with respect to Euclidean distance.
For $\eps \in (0,1]$ and $n \geq 2k$,
\[
N(U_k, \Norm{\cdot}_2, \eps) \geq \left( \frac{n-k}{k} \right)^{k \left( 1-\frac{\eps^2}{2} \right)}.
\]
\end{restatable}

\begin{restatable}[]{lem}{KLupperbound}
\label{lem:KL-upperbound}
Denote $\cS^{n-1}_k$ as the set of $k$-sparse unit vectors. Then,
\[
\max_{u, v \in \cS^{n-1}_k} D_{KL} \left( \bbP_{\pmb Y \sim \cY \mid u} \Bigm\Vert \bbP_{\pmb Y \sim \cY \mid v} \right)
\leq 2 \lambda^2
\]
where $D_{KL}(\cdot \Vert \cdot)$ is the KL-divergence function and $\bbP_{\pmb Y \sim \cY \mid u}$ is the probability distribution of observing $\pmb Y$ from signal $u$ with additive standard Gaussian noise tensor $\pmb W$.
\end{restatable}

We are now ready to prove \cref{thm:minimax-bound} by using the above two lemmata.

\begin{proof}[Proof of \cref{thm:minimax-bound}]
The theorem follows by computing a lower bound for \cref{eqn:info-lb-eqn} with $\cX$ as an $\eps$-packing of $U_k$ of size $\Abs{\cX} = m \geq P(U_k, \rho, \eps)$.

\cref{lem:rademacher-eps-net-lower-bound} tells us that $N(U_k, \Norm{\cdot}_2, \eps) \geq \left( \frac{n-k}{k} \right)^{k \left( 1-\frac{\eps^2}{2} \right)}$.
Since $\Norm{x-x'}_2 \geq \rho(x,x')$, we see that $N(U_k, \Norm{\cdot}_2, \eps) \leq P(U_k, \Norm{\cdot}_2, \eps) \leq P(U_k, \rho, \eps) \leq m$.
Thus,
\[
k \left( 1-\frac{\eps^2}{2} \right) \log \left( \frac{n-k}{k} \right) \leq \log m
\]

Meanwhile, \cref{lem:KL-upperbound} tells us that $\max_{u, v \in \cS^{n-1}_k} D_{KL} \left( \bbP_{\pmb Y \sim \cY \mid u} \Bigm\Vert \bbP_{\pmb Y \sim \cY \mid v} \right) \leq 2 \lambda^2$.
Since $U_k \subseteq \cS^{n-1}_k$, this implies that
\[
\max_{u, v \in \cX} D_{KL} \left( \bbP_{\pmb Y \sim \cY \mid u} \Bigm\Vert \bbP_{\pmb Y \sim \cY \mid v} \right) \leq 2 \lambda^2
\]

Let $\tau > 0$ be a lower bound constant (which we fix later).
Putting the above bounds together, we have
\begin{align*}
\inf_{\widehat{x} \in U_k} \sup_{x \in U_k} \E_{\pmb Y} \left[ 1 - \Abs{\langle \widehat{x}, x \rangle} \right]
\geq & \frac{\eps^2}{4} \cdot \left( 1 - \frac{\max_{u, v \in \cX} D_{KL} \left( \bbP_{\pmb Y \sim \cY \mid u} \Bigm\Vert \bbP_{\pmb Y \sim \cY \mid v} \right) + 1}{\log m} \right) && \text{\cref{eqn:info-lb-eqn}}\\
\geq & \frac{\eps^2}{4} \cdot \left( 1 - \frac{2 \lambda^2 + 1}{k \left( 1-\frac{\eps^2}{2} \right) \log \left( \frac{n-k}{k} \right)} \right)\\
\geq & \tau
\end{align*}
Rearranging, we get $\lambda \leq \sqrt{\frac{(1 - (4 \tau)/\eps^2) (1 - \eps^2 / 2)}{2} k \log \left( \frac{n-k}{k} \right) - \frac{1}{2}}$.
The claims follows\footnote{Observe that $(1-(4*0.05)/(0.5^2))*(1-(0.5^2)/2)/2 = 0.0875 > 1/12$.} by setting $\tau = 0.05$ and $\eps = 1/2$.
\end{proof}

In the setting context of our interest, the works of \cite{NEURIPS2020_cd0b43ea,perry2020statistical} both papers outline a phase transition at $\lambda =\Theta(\sqrt{k \log (n/k)})$.
Specifically, they prove that weak recovery is possible when $\lambda \gtrsim \sqrt{k \log (n/k)}$ and information theoretically impossible when $\lambda \lesssim \sqrt{k \log (n/k)}$.
Our information theoretic bounds presented here are equivalent to these results, up to constant factors.

\section*{Acknowledgements}
We thank David Steurer for several helpful conversations.
We thank Luca Corinzia and Paolo Penna for useful discussions about the planted sparse densest sub-hypergraph model.

\bibliography{refs}
\bibliographystyle{alpha}

\appendix
\section{Related Model: Planted sparse densest sub-hypergraph}
\label{sec:PSDM}

The planted $k$-densest sub-hypergraph model (\cite{corinzia2019exact, buhmannrecovery, corinzia2020statistical}) is closely related to our sparse spiked tensor model \cref{def:sstm}.
While not directly reducible from/to one another, techniques developed in one model can inform the other.

The planted $k$-densest sub-hypergraph model is a weighted complete hypergraph where a subset of $k$ planted vertices, denoted by $S \subseteq [n]$, is \emph{drawn uniformly at random} and each hyperedge involves $p$ vertices, for $2 \leq p \leq k \leq n$.
Except for the $\binom{k}{p}$ \emph{one-sided biased} hyperedges (belonging to the planted subgraph induced by the $k$ vertices in $S$) whose weights follow the Gaussian distribution $N(\beta,\sigma^2)$, the weight of all remaining $\binom{n}{p}-\binom{k}{p}$ hyperedges follow a Gaussian distribution $N(0,\sigma^2)$.
In other words, the hyperedge defined by $\{i_1, \ldots, i_p\}$, for $i_1, \ldots, i_p \in [n]$, has weight
\[
\pmb Y_{i_1, \ldots, i_p} =
\begin{cases}
	\beta + \pmb W_{i_1, \ldots, i_p} & \text{if $i_1, \ldots, i_p \in S$}\\
	\pmb W_{i_1, \ldots, i_p} & \text{otherwise}
\end{cases}
\]
where each $\pmb W_{i_1, \ldots, i_p} \sim N(0,\sigma^2)$ is independent and planted entries have a $\beta > 0$ bias.

As one can see, the planted $k$-densest sub-hypergraph model (PDSM) is very similar to the single-spike ($r=1$) sparse spiked tensor model (SSTM) that we study\footnote{We believe that handling a more general $\sigma^2$ is a non-issue when comparing these models because the $\sigma^2$ factor could be propagated throughout our analysis by appropriately adjusting the sub-Gaussian concentration arguments.}.
However, there are two key model differences that one needs to be aware of.
Firstly, there are $n^p$ observations in SSTM instead of $\binom{n}{p}$ in PDSM as the former is not constrained to hyperedges (e.g.\ $\pmb Y_{1, \ldots, 1}$ is a valid data observation in SSTM but not in PDSM).
Secondly, our signal bias is not one-sided and is scaled by a factor of $k^{-p/2}$: For a planted coordinate $(i_1, \ldots, i_p)$, SSTM observes $\pmb Y_{i_1, \ldots, i_p} = \pmb W_{i_1, \ldots, i_p} \pm \lambda k^{-p/2}$ instead of $\pmb Y_{i_1, \ldots, i_p} = \pmb W_{i_1, \ldots, i_p} + \beta$ in PDSM, where the sign of bias in SSTM depends on the polarities of signal entries $x_{i_1}, \ldots, x_{i_p}$.

While the signal scaling discrepancies can be handled by replacing $\beta \sqrt{\binom{k}{p}}$ terms in PDSM bounds with $\lambda$ in SSTM\footnote{This discrepancy arises due to having $\binom{k}{p}$ planted hyperedges in PDSM, as opposed to $k^p$ signal entries in SSTM, and the signal strength scaling of $k^{-p/2}$ in SSTM.}, the one-sidedness of the signal bias has implications on the computational hardness of the two models.
In a recent work, \cite{corinzia2020statistical} showed that an Approximate Message Passing (AMP) algorithm succeeds in signal recovery for PDSM when $\lambda \gtrsim \frac{k}{p \sqrt{n}} \left(\frac{n}{k}\right)^{p/2}$.
In constrast, \cref{thm:main-algorithm-multi-informal} (for constant $p$) tells us that our polynomial time algorithm for SSTM succeeds when $k \leq \sqrt{np}$ and $\lambda \gtrsim \sqrt{k^p \log \left(\frac{np}{k}\right)}$.
Meanwhile, \cref{thm:main-lower-bound-informal} implies that it is impossible to recover the signal using low-degree polynomials whenever the signal-to-noise ratio satisfies $\lambda \lesssim \widetilde{O} \left( \min \left\{ \left(\frac{n}{p}\right)^{p/4}, \left( \frac{k}{p} \left(1 + \Abs{\ln\left( \frac{np}{k^2} \right)} \right) \right)^{p/2} \right\} \right)$.
Indeed, the one-sidedness of the signal bias in PDSM appears to make the problem \emph{computationally easier} than SSTM in some regimes\footnote{e.g.\ Large $k$ regimes such as $k = n^{0.9}$. For large $k$, a heuristic adaptation of our low-degree analysis to PDSM shows that the relationship between parameters $\lambda$, $n$, $k$ and $p$ in a low-degree bound is roughly of the form $\lambda \gtrsim \left(\frac{n}{k}\right)^{p/2}$ instead of $\lambda \gtrsim n^{p/4}$. This roughly matches the AMP bounds shown by \cite{corinzia2020statistical} and further provides credence to our claim that techniques from one model can applied to the other.}.

Nevertheless, we believe that techniques used in either model are generally applicable in the other and we expect a variant of our limited brute force algorithm to work in PDSM.
From a statistical viewpoint, \cite{corinzia2020statistical} proved information-theoretic lower bounds for recovery in PDSM of $\lambda \lesssim \sqrt{k \log n}$ while we have $\lambda \lesssim \sqrt{k \log ((n-k)/k) }$ for approximate signal recovery\footnote{i.e.\, It is enough to find a strongly correlated estimate $\widehat{x}$ of the signal $x$ where $\widehat{x}$ could be ``wrong on a few coordinates''. The bounds for exact and approximate recovery in \cite{corinzia2020statistical} differ by constant factors.} in SSTM, which matches the PDSM bounds when $k \in o(n)$.
These results in both models rely on standard techniques such as Fano's inequality.

\section{Deferred proofs and details}
\label{sec:missing-proofs}

This section provides the formal proofs that were deferred in favor for readability.
For convenience, we will restate the statements before proving them.

\subsection{Sparse norm bounds}
\label{sec:sparse-norm-bounds-proof}

\tensorsparsebound*
\begin{proof}[Proof of \cref{lem:tensor-sparse-bound}]
We will focus on proving the following statement:
\begin{equation}
\label{eqn:max-bound}
\bbP \left[ \max_{\substack{x_{(1)}, \ldots, x_{(p)} \in \cS^{n - 1}_s,\\ \text{$r$ distinct vectors}}} \pmb T(x_{(1)}, \ldots, x_{(p)}) \geq \sqrt{8 \cdot \left( 4rs \ln \left( \frac{np}{s} \right) + \ln \left( \frac{1}{\gamma} \right) \right)} \right] \leq \gamma
\end{equation}
By a similar argument, one can obtain
\begin{equation}
\label{eqn:min-bound}
\bbP \left[ \min_{\substack{x_{(1)}, \ldots, x_{(p)} \in \cS^{n - 1}_s,\\ \text{$r$ distinct vectors}}} \pmb T \left( x_{(1)}, \ldots, x_{(p)} \right) \leq -\sqrt{8 \cdot \left( 4rs \ln \left( \frac{np}{s} \right) + \ln \left( \frac{1}{\gamma} \right) \right)} \right]
\leq \gamma 
\end{equation}
The lemma follows from a union bound of \cref{eqn:max-bound} and \cref{eqn:min-bound}.

It now remains to prove \cref{eqn:max-bound}.
Let $\lambda, t, \eps$ be proof parameters which we fix later.
Define $\cN(\cS^{n-1}_s)$ as an $\eps$-cover of $\cS^{n-1}_s$ of smallest cardinality $N(\cS^{n-1}_s, \epsilon)$.
It is known\footnote{e.g.\ See \cite[Corollary 4.2.13]{vershynin2018high}.} that $\left( \frac{1}{\epsilon} \right)^n \leq N(\cS^{n-1}, \epsilon) \leq \left( \frac{2}{\epsilon} + 1 \right)^n \leq \left( \frac{3}{\epsilon} \right)^n$.
Treating each unit sphere defined on $s$ coordinates independently and then taking union bound gives us $N(\cS^{n-1}_s, \epsilon) \leq \binom{n}{s} \cdot N(\cS^{n-1}, \epsilon) \leq \left( \frac{en}{s} \right)^s \left( \frac{3}{\epsilon} \right)^s$.
So, $\Abs{\cN(\cS^{n-1}_s)} \leq \left( \frac{3en}{\eps s} \right)^s$.
For any $r$ distinct vectors $x_{(1)}, \ldots, x_{(p)} \in \cN(\cS^{n-1}_s)$,
\begin{align*}
&\; \bbP \left[ \pmb T(x_{(1)}, \ldots, x_{(p)}) \geq t \right]\\
= &\; \bbP \left[ \sum_{i_1, \ldots, i_p = 1}^n \pmb T_{i_1, \ldots, i_p} x_{(1), i_1} x_{(2), i_2} \ldots x_{(p), i_p} \geq t \right]\\
= &\; \bbP \left[ \exp \left( \lambda \sum_{i_1, \ldots, i_p = 1}^n \pmb T_{i_1, \ldots, i_p} x_{(1), i_1} x_{(2), i_2} \ldots x_{(p), i_p} \right) \geq e^{\lambda t} \right]\\
\leq &\; e^{-\lambda t} \cdot \E \left[ \exp \left(\lambda \sum_{i_1, \ldots, i_p = 1}^n \pmb T_{i_1, \ldots, i_p} x_{(1), i_1} x_{(2), i_2} \ldots x_{(p), i_p} \right) \right] && \text{Markov's inequality}\\
= &\; e^{- \lambda t} \cdot \Pi_{i_1, \ldots, i_p = 1}^n \E \left[ \exp \left( \pmb T_{i_1, \ldots, i_p} \lambda x_{(1), i_1} x_{(2), i_2} \ldots x_{(p), i_p} \right) \right]\\
= &\; e^{- \lambda t} \cdot \Pi_{i_1, \ldots, i_p = 1}^n \exp \left( \frac{ \left( \lambda x_{(1), i_1} x_{(2), i_2} \ldots x_{(p), i_p} \right)^2}{2} \right) && \text{$\pmb T_{i_1, \ldots, i_p} \sim N(0,1)$}\\
= &\; \exp \left(- \lambda t + \sum_{i_1, \ldots, i_p = 1}^n \frac{ \left( \lambda x_{(1), i_1} x_{(2), i_2} \ldots x_{(p), i_p} \right)^2}{2} \right)\\
= &\; \exp \left(- \lambda t + \frac{\lambda^2}{2} \right) && \sum_{i_1=1}^n x_{(1), i_1}^2 = \ldots = \sum_{i_p=1}^n x_{(p), i_p}^2 = 1\\
\leq &\; \exp \left(- \frac{t^2}{2} \right) && \text{Maximized when $\lambda = t$}
\end{align*}

By union bound over all $\left( N(\cS^{n-1}_s, \eps) \right)^r \leq \left( \frac{3en}{\eps s} \right)^{rs}$ $r$ distinct points in $\otimes^p \cN(\cS^{n-1}_s)$,
\begin{align*}
\bbP \left[ \max_{\substack{x_{(1)}, \ldots, x_{(p)} \in \cN(\cS^{n-1}_s),\\ \text{$r$ distinct vectors}}} \pmb T(x_{(1)}, \ldots, x_{(p)}) \geq t \right]
& \leq \sum_{\substack{x_{(1)}, \ldots, x_{(p)} \in \cN(\cS^{n-1}_s),\\ \text{$r$ distinct vectors}}} \bbP \left[ \pmb T(x_{(1)}, \ldots, x_{(p)}) \geq t \right]\\
& \leq \left( \frac{3en}{\eps s} \right)^{rs} \exp \left(- \frac{t^2}{2} \right)\\
& = \exp \left( rs \ln \left( \frac{3en}{\eps s} \right) - \frac{t^2}{2} \right)
\end{align*}

As $\otimes^p \cS^{n-1}_s$ is compact, there are $r$ distinct vectors $x_{(1)}^*, x_{(2)}^*, \ldots, x_{(p)}^* \in \cS^{n-1}_s$ such that
\[
\left( x_{(1)}^*, x_{(2)}^*, \ldots, x_{(p)}^* \right)
= \argmax_{\substack{x_{(1)}, \ldots, x_{(p)} \in \cS^{n - 1}_s,\\ \text{$r$ distinct vectors}}} \pmb T \left( x_{(1)}, \ldots, x_{(p)} \right)
\]

By definition of $\eps$-cover, there are vectors $x_{(1)}, x_{(2)}, \ldots, x_{(p)} \in \cN(\cS^{n-1}_s)$ such that $x_{(1)}^* = x_{(1)} + \delta_{(1)}, \ldots, x_{(p)}^* = x_{(p)} + \delta_{(p)}$, where $\Norm{\delta_{(z)}}_2 \leq \eps$ for $z \in \{1, \ldots, p\}$.
Let $z \in \{1, \ldots, p\}$.
Since $x_{(z)}^*$ and $x_{(z)}$ are $s$-sparse, $\delta_{(z)}$ is at most $(2s)$-sparse.
We can express $\delta_{(z)} = \delta^{(1)}_{(z)} + \delta^{(2)}_{(z)}$ as a sum of two $s$-sparse vectors where $\frac{\delta^{(1)}_{(z)}}{\Norm{\delta^{(1)}_{(z)}}_2}, \frac{\delta^{(2)}_{(z)}}{\Norm{\delta^{(2)}_{(z)}}_2} \in \cS^{n-1}_s$, $\Norm{\delta^{(1)}_{(z)}}_2 \leq \Norm{\delta_{(z)}}_2 \leq \eps$, and $\Norm{\delta^{(2)}_{(z)}}_2 \leq \Norm{\delta_{(z)}}_2 \leq \eps$.
We can relate $\pmb T(x_{(1)}^*, x_{(2)}^*, \ldots, x_{(p)}^*)$ to $\pmb T(x_{(1)}, x_{(2)}, \ldots, x_{(p)})$ by expanding the definition:
\begin{align*}
&\; \max_{\substack{x_{(1)}, \ldots, x_{(p)} \in \cS^{n - 1}_s,\\ \text{$r$ distinct vectors}}} \pmb T(x_{(1)}, x_{(2)}, \ldots, x_{(p)})\\
= &\; \sum_{i_1, \ldots, i_p = 1}^n \pmb T_{i_1, \ldots, i_p} x_{(1), i_1}^* x_{(2), i_2}^* \ldots x_{(p), i_p}^*\\
= &\; \sum_{i_1, \ldots, i_p = 1}^n \pmb T_{i_1, \ldots, i_p} \left( x_{(1)} + \delta^{(1)}_{(1)} + \delta^{(2)}_{(1)} \right)_{i_1} \left( x_{(2)} + \delta^{(1)}_{(2)} + \delta^{(2)}_{(2)} \right)_{i_2} \ldots \left( x_{(p)} + \delta^{(1)}_{(p)} + \delta^{(2)}_{(p)} \right)_{i_p}\\
\leq &\; \pmb T(x_{(1)}, \ldots, x_{(p)}) + \left( \max_{\substack{x_{(1)}, \ldots, x_{(p)} \in \cS^{n - 1}_s,\\ \text{$r$ distinct vectors}}} \pmb T(x_{(1)}, \ldots, x_{(p)}) \right) \cdot \left(\eps \cdot 2 \binom{p}{1} + \eps^2 \cdot 2^2 \binom{p}{2} + \ldots + \eps^p 2^p \binom{p}{p} \right)\\
\leq &\; \pmb T(x_{(1)}, \ldots, x_{(p)}) + \left( \max_{\substack{x_{(1)}, \ldots, x_{(p)} \in \cS^{n - 1}_s,\\ \text{$r$ distinct vectors}}} \pmb T(x_{(1)}, \ldots, x_{(p)}) \right) \cdot \left( \frac{2p \eps}{1!} + \frac{(2p \eps)^2}{2!} + \ldots + \frac{(2p \eps)^p}{p!} \right)\\
\leq &\; \pmb T(x_{(1)}, \ldots, x_{(p)}) + \left( \max_{\substack{x_{(1)}, \ldots, x_{(p)} \in \cS^{n - 1}_s,\\ \text{$r$ distinct vectors}}} \pmb T(x_{(1)}, \ldots, x_{(p)}) \right) \cdot \left( e^{2p \eps} - 1 \right)
\end{align*}

The first inequality is by counting how the $\delta$'s group together, factoring out their norms so that they belong to $\cS^{n-1}_s$ (so $\max_{(\ldots)} \pmb T \left(x_{(1)}, \ldots, x_{(p)} \right)$ applies), then using $\Norm{\delta^{(1)}_{(z)}}_2, \Norm{\delta^{(2)}_{(z)}}_2 \leq \eps$.
The second inequality is due to $\binom{n}{k} \leq \frac{n^k}{k!}$.
The third is due to the definition of $e^x = \sum_{n=0}^{\infty} \frac{x^n}{n!}$.
Note that $1 > e^{2p \eps} - 1$ if and only if $\eps < \frac{\ln 2}{2p}$.
Set $\eps = \frac{\ln 2}{4p}$, then $\frac{1}{2 - e^{2p\eps}} < 2$.
Rearranging, we get
\begin{align*}
\max_{\substack{x_{(1)}, \ldots, x_{(p)} \in \cS^{n - 1}_s,\\ \text{$r$ distinct vectors}}} \pmb T \left( x_{(1)}, x_{(2)}, \ldots, x_{(p)} \right)
& \leq \frac{\pmb T \left( x_{(1)}, x_{(2)}, \ldots, x_{(p)} \right)}{2 - e^{2p \eps}}\\
& \leq 2 \max_{\substack{x_{(1)}, \ldots, x_{(p)} \in \cN(\cS^{n-1}_s),\\ \text{$r$ distinct vectors}}} \pmb T \left( x_{(1)}, x_{(2)}, \ldots, x_{(p)} \right)
\end{align*}

Thus,
\begin{align*}
&\; \bbP \left[ \max_{\substack{x_{(1)}, \ldots, x_{(p)} \in \cS^{n - 1}_s,\\ \text{$r$ distinct vectors}}} \pmb T \left( x_{(1)}, x_{(2)}, \ldots, x_{(p)} \right) \geq t \right]\\
\leq &\; \bbP \left[\max_{\substack{x_{(1)}, \ldots, x_{(p)} \in \cN(\cS^{n-1}_s),\\ \text{$r$ distinct vectors}}} \pmb T \left( x_{(1)}, x_{(2)}, \ldots, x_{(p)} \right) \geq \frac{t}{2} \right]\\
\leq &\; \exp \left( rs \ln \left( \frac{3en}{\eps s} \right) - \frac{(t/2)^2}{2} \right) && \text{From above}\\
\leq &\; \exp \left( 4rs \ln \left( \frac{np}{s} \right) - \frac{t^2}{8} \right) && \text{Since $\eps = \frac{\ln 2}{4p}$ and $\ln \left( \frac{12e}{\ln 2} \right) < 4$}
\end{align*}

Setting $t^2 = 8 \cdot \left( 4rs \ln \left( \frac{np}{s} \right) + \ln \left( \frac{1}{\gamma} \right) \right)$ yields \cref{eqn:max-bound}.
\end{proof}

\subsection{Proofs for recovery algorithms}
\label{sec:missing-recovery-proofs}

\apxflatmaximizer*
\begin{proof}[Proof of \cref{lem:apxflat-maximizer}]
Without loss of generality, suppose that $x_{(1)}, \ldots, x_{(s)}$ are the remaining $s$ (where $1 \leq s \leq r$) unrecovered signals with signal strengths $\lambda_1, \ldots, \lambda_s$ such that $\lambda_1 \geq \ldots \geq \lambda_s \geq \lambda_r$.
Let $u_* \in U_t$ lie completely in some signal $\hat{x}$ with signal strength $\hat{\lambda}$ such that
\[
\hat{\lambda} \iprod{\hat{x}, u_*}^p
\geq \max_{q \in [s]} \max_{u \in U_t} \lambda_q \iprod{x_{(q)}, u}^p
\quad \text{and} \quad
\iprod{\pmb Y^{(1)}, \tensorpower{u_*}{p}}
\geq \frac{\hat{\lambda}}{\sqrt{2}} \iprod{\hat{x},u_*}^p + \iprod{\pmb W^{(1)}, \tensorpower{u_*}{p}}.
\]

By optimality, $\iprod{\pmb Y^{(1)}, \tensorpower{\pmb v_*}{p}} \geq \iprod{\pmb Y^{(1)}, \tensorpower{u_*}{p}}$.
So, the claim holds if we can show that $\iprod{\pmb Y^{(1)}, \tensorpower{u_*}{p}} > \iprod{\pmb Y^{(1)}, \tensorpower{u}{p}}$ for \emph{any} $u \in U_t$ such that
\begin{equation}
\label{eq:violate-claim-apxflat}
\Abs{\supp\Paren{u} \cap \supp\Paren{x_{(1)}}} < (1-\eps) \cdot t, \ldots, \Abs{\supp\Paren{u} \cap \supp\Paren{x_{(s)}}} < (1-\eps) \cdot t\,.
\end{equation}

For any $u \in U_t$ that satisfies \cref{eq:violate-claim-apxflat}, we see that
\begin{align*}
&\; \iprod{\pmb Y^{(1)}, \tensorpower{u}{p}}\\
= &\; \iprod{\pmb W^{(1)}, \tensorpower{u}{p}} + \sum_{q=1}^s \frac{\lambda_q}{\sqrt{2}} \iprod{u, x_{(q)}}^{p}\\
\leq &\; \iprod{\pmb W^{(1)}, \tensorpower{u}{p}} + \frac{\hat{\lambda}}{\sqrt{2}} \Paren{\iprod{\hat{x},u_*} - \frac{\eps}{A}\sqrt{\frac{t}{k}}}^p + \frac{\lambda_1 \eps^p A^p}{\sqrt{2}} \Paren{\frac{t}{k}}^{\frac{p}{2}} && \text{\cref{eq:violate-claim-apxflat}}\\
\leq &\; \iprod{\pmb W^{(1)}, \tensorpower{u}{p}} + \frac{\hat{\lambda}}{\sqrt{2}} \Paren{\Paren{\iprod{\hat{x},u_*} - \frac{\eps}{A}\sqrt{\frac{t}{k}}}^p + \frac{\eps^p A^p}{\kappa} \Paren{\frac{t}{k}}^{\frac{p}{2}}} && \text{$\hat{\lambda} \geq \lambda_r \geq \kappa \lambda_1$}\\
\leq &\; \iprod{\pmb W^{(1)}, \tensorpower{u}{p}} + \frac{\hat{\lambda}}{\sqrt{2}} \iprod{\hat{x},u_*}^p \Paren{\Paren{1 - \frac{\eps}{A^2}}^p + \frac{\eps^p A^p}{\kappa}} && \text{$\sqrt{\frac{t}{k}} \leq \iprod{\hat{x},u_*} \leq A \sqrt{\frac{t}{k}}$}\\
\leq &\; \iprod{\pmb W^{(1)}, \tensorpower{u}{p}} + \frac{\hat{\lambda}}{\sqrt{2}} \iprod{\hat{x},u_*}^p \Paren{\Paren{1 - \frac{\eps}{A^2}}^p + \frac{\eps (1-\eps)^{p-1}}{A^p}} && \text{$\kappa \geq A^{2p} \Paren{\frac{\eps}{1-\eps}}^{p-1}$}\\
\leq &\; \iprod{\pmb W^{(1)}, \tensorpower{u}{p}} + \frac{\hat{\lambda}}{\sqrt{2}} \iprod{\hat{x},u_*}^p \Paren{1 - \frac{\eps}{A^2}}^{p-1} && \text{$A \geq 1$, $\eps \leq \frac{1}{2}$}\\
\end{align*}

Let us set parameters $\Paren{r, s, \gamma}$ as $\Paren{1, t, \exp \Paren{- \frac{\lambda_r^2 \eps^2}{128 A^4} \Paren{\frac{t}{k}}^p}}$ in \cref{lem:tensor-sparse-bound}.
Since
\[
\frac{\lambda_r^2 \eps^2}{128 A^4} \Paren{\frac{t}{k}}^p
\geq 8t \ln(n)
\geq 4t \ln \Paren{\frac{np}{t}},
\]
we see that
\[
\frac{\lambda_r \eps}{2 \sqrt{2} A} \Paren{\frac{t}{k}}^{\frac{p}{2}}
\geq \sqrt{8 \Paren{4t \ln \Paren{\frac{np}{t}} + \ln \Paren{\frac{1}{\gamma}}}}\,.
\]
Thus, \cref{lem:tensor-sparse-bound} gives us that, for any $u \in U_t$,
\[
\bbP \left[ \max_{u \in U_t} \Abs{\iprod{\pmb W^{(1)}, \tensorpower{u}{p}}} \geq \frac{\lambda_r \eps}{2 \sqrt{2} A^2} \Paren{\frac{t}{k}}^{\frac{p}{2}} \right]
\leq 2 \exp \Paren{-\frac{\lambda_r^2 \eps^2}{128 A^4} \Paren{\frac{t}{k}}^p}\,.
\]
Then, with probability at least $1 - 4 \exp \exp \Paren{-\frac{\lambda_r^2 \eps^2}{128 A^4} \Paren{\frac{t}{k}}^p}$,
\[
\iprod{\pmb W^{(1)}, \tensorpower{u_*}{p}} - \iprod{\pmb W^{(1)}, \tensorpower{u}{p}}
< \frac{\hat{\lambda}}{\sqrt{2}} \iprod{\hat{x},u_*}^p \frac{\eps}{A^2}
\leq \frac{\hat{\lambda}}{\sqrt{2}} \iprod{\hat{x},u_*}^p \Paren{1 - \Paren{1- \frac{\eps}{A^2}}^{p-1}}\,.
\]
and so
$
\iprod{\pmb Y^{(1)}, \tensorpower{\pmb v_*}{p}}
> \iprod{\pmb Y^{(1)}, \tensorpower{u_*}{p}}
> \iprod{\pmb Y^{(1)}, \tensorpower{u}{p}}\,.
$
for any $u \in U_t$ that satisfies \cref{eq:violate-claim-apxflat}.
\end{proof}

\apxflatanchor*
\begin{proof}[Proof of \cref{lem:apxflat-anchor}]
Recall that
\[
\pmb \alpha_\ell
=  \underset{q \in [r]}{\sum} \frac{\lambda_q}{\sqrt{2}} x_{(q), \ell} \iprod{x_{(q)}, \pmb v_*}^{p-1} + \iprod{\pmb W^{(2)}, \tensorpower{\pmb v_*}{p-1} \otimes e_\ell}.
\]
Since $\pmb W^{(2)}$ is independent from $\pmb W^{(1)}$, we can apply standard Gaussian bounds.
That is,
\[
\Pr \Brac{\Abs{\iprod{\pmb W^{(2)}, \tensorpower{\pmb v_*}{p-1} \otimes e_\ell}} \geq \lambda_r \frac{A^p \eps^{p-1}}{2 \kappa \sqrt{2t}} \Paren{\frac{t}{k}}^{\frac{p}{2}}}
\leq 2 \exp \Paren{- \lambda_r^2 \frac{A^{2p} \eps^{2p-2}}{16 \kappa^2 t} \Paren{\frac{t}{k}}^{p}}\,.
\]
Now, conditioned on
\[
\Abs{\iprod{\pmb W^{(2)}, \tensorpower{\pmb v_*}{p-1} \otimes e_\ell}}
< \lambda_r \frac{A^p \eps^{p-1}}{2 \kappa \sqrt{2k}} \Paren{\frac{t}{k}}^{\frac{p-1}{2}}
< \lambda_r \frac{(1-\eps)^{p-1}}{2 A^p \sqrt{2k}} \Paren{\frac{t}{k}}^{\frac{p-1}{2}},
\]
we consider cases of $\ell \in \supp\Paren{x_{(\pi(i))}}$ and $\ell \not\in \supp\Paren{x_{(\pi(i))}}$ separately.
To be precise, we will show the following two results:
\begin{enumerate}
\item $\bbP \left[ \Abs{\pmb \alpha_\ell} < \lambda_r \frac{(1-\eps)^{p-1}}{2 A^{p} \sqrt{2k}} \Paren{\frac{t}{k}}^{\frac{p-1}{2}} \Bigm\vert \ell \in \supp\Paren{x_{(\pi(i))}} \right]
\leq 2 \exp \Paren{- \lambda_r^2 \frac{A^{2p} \eps^{2p-2}}{16 \kappa^2 t} \Paren{\frac{t}{k}}^{p}}$
\item $\bbP \left[ \Abs{\pmb \alpha_\ell} > \lambda_r \frac{2 A^p \eps^{p-1}}{\kappa \sqrt{2k}} \Paren{\frac{t}{k}}^{\frac{p-1}{2}} \Bigm\vert \ell \not\in \supp\Paren{x_{(\pi(i))}} \right]
\leq 2 \exp \Paren{- \lambda_r^2 \frac{A^{2p} \eps^{2p-2}}{16 \kappa^2 t} \Paren{\frac{t}{k}}^{p}}$
\end{enumerate}
As $\kappa > 4 A^{2p} \Paren{\frac{\eps}{1-\eps}}^{p-1}$, there will be a value gap in $\Abs{\pmb \alpha_\ell}$ for $\ell \in \supp\Paren{x_{(\pi(i))}}$ versus $\ell \not\in \supp\Paren{x_{(\pi(i))}}$.
The result follows by taking a union bound over all $n$ coordinates.

\paragraph{Case 1 $\left( \ell \in \supp\Paren{x_{(\pi(i))}} \right)$:}
Since $\Abs{\supp\Paren{\pmb v_*} \cap \supp\Paren{x_{(\pi(i))}}} \geq (1 - \eps) \cdot t$,
\[
\Abs{\frac{\lambda_{\pi(i)}}{\sqrt{2}} \cdot x_{(\pi(i)),\ell} \cdot \iprod{\pmb v_*, x_{(\pi(i))}}^{p-1}}
\geq \frac{\lambda_r}{A \sqrt{2k}} \cdot \Abs{\iprod{\pmb v_*, x_{(\pi(i))}}^{p-1}}
\geq \lambda_r \frac{(1-\eps)^{p-1}}{A^p \sqrt{2k}} \Paren{\frac{t}{k}}^{\frac{p-1}{2}}\,.
\]

By reverse triangle inequality, we have
\[
\Abs{\pmb \alpha_\ell}
= \lambda_r \frac{(1-\eps)^{p-1}}{A^p \sqrt{2k}} \Paren{\frac{t}{k}}^{\frac{p-1}{2}} - \Abs{\iprod{\pmb W^{(2)}, \tensorpower{\pmb v_*}{p-1} \otimes e_\ell}}
> \lambda_r \frac{(1-\eps)^{p-1}}{2A^p \sqrt{2k}} \Paren{\frac{t}{k}}^{\frac{p-1}{2}}\,.
\]

\paragraph{Case 2 $\left( \ell \not\in \supp\Paren{x_{(\pi(i))}} \right)$:}
Since signals have disjoint support and $\Abs{\supp\Paren{\pmb v_*} \cap \supp\Paren{x_{(\pi(i))}}} \geq (1 - \eps) \cdot t$, we have $\Abs{\supp\Paren{\pmb v_*} \cap \supp\Paren{x_{(j)}}} < \eps \cdot t$.

By triangle inequality, we have
\begin{align*}
\Abs{\pmb \alpha_\ell}
&\leq \lambda_1 \frac{A^p \eps^{p-1}}{\sqrt{2k}} \Paren{\frac{t}{k}}^{\frac{p-1}{2}} + \Abs{\iprod{\pmb W^{(2)}, \tensorpower{\pmb v_*}{p-1} \otimes e_\ell}}\\
&\leq \lambda_r \frac{A^p \eps^{p-1}}{\kappa \sqrt{2k}} \Paren{\frac{t}{k}}^{\frac{p-1}{2}} + \Abs{\iprod{\pmb W^{(2)}, \tensorpower{\pmb v_*}{p-1} \otimes e_\ell}}\\
&\leq \lambda_r \frac{2 A^p \eps^{p-1}}{\kappa \sqrt{2k}} \Paren{\frac{t}{k}}^{\frac{p-1}{2}}\,.
\end{align*}
\end{proof}

\subsection{Proofs for computational bounds}
\label{sec:comp-bound-proofs}

\cref{claim:counting} relates the counting of $\pmb Y$ entries with coordinates of the signal $\pmb x$.
In the claim, $s \in [n]$ is the number of entries of $\pmb x$ that is considered in the summation.
We only need to consider $s$ up to $\lfloor pd/2 \rfloor$ because the expectation is 0 if some coordinate of $\pmb x$ is used an odd number of times.
Each $\alpha$ can be viewed as $d$ consecutive chunks of $p$ entries, and each $(\beta_1, \ldots, \beta_s)$ counts the number of times $\pmb x_j$ occurs in $\alpha$.

\begin{claim}
\label{claim:counting}
For a fixed degree $d \leq 2n/p$,
\begin{align*}
& \sum_{\Abs{\alpha} = d} \mathbbm{1}_{even(c(\alpha))} \left( \frac{k}{n} \right)^{2 s(\alpha)} \left( \Pi_{i=1}^{n^p} \frac{1}{(\alpha_i)!} \right)\\
= & \sum_{s=1}^{\left\lfloor pd/2 \right\rfloor} \binom{n}{s} \left( \frac{k}{n} \right)^{2s} \sum_{\substack{\beta_1 + \ldots + \beta_s = pd/2\\ \beta_1 \neq 0, \ldots, \beta_s \neq 0}} \binom{pd}{2 \beta_1, \ldots, 2 \beta_s} \frac{1}{\binom{d}{\alpha_1, \ldots, \alpha_{n^p}}} \left( \Pi_{i=1}^{n^p} \frac{1}{(\alpha_i)!} \right)\\
= & \frac{1}{d!} \sum_{s=1}^{\left\lfloor pd/2 \right\rfloor} \binom{n}{s} \left( \frac{k}{n} \right)^{2s} \sum_{\substack{\beta_1 + \ldots + \beta_s = pd/2\\ \beta_1 \neq 0, \ldots, \beta_s \neq 0}} \binom{pd}{2 \beta_1, \ldots, 2 \beta_s}
\end{align*}
\end{claim}
\begin{proof}
The second equality is by definition of multinomial coefficients.
To prove the first equality, we consider two equivalent ways of viewing the summation.

From the viewpoint of choosing entries of $\pmb Y$, one chooses $d$ (possibly repeated) entries of $\pmb Y$ and computes $\mathbbm{1}_{even(c(\alpha))} \left( \frac{k}{n} \right)^{2 s(\alpha)} \left( \Pi_{i=1}^{n^p} \frac{1}{(\alpha_i)!} \right)$ directly on the corresponding $\alpha$.

From the viewpoint of choosing entries from $\pmb x$, first observe that each $\alpha$ considered actually involves $pd$ (possibly repeated) entries of $[n]$ and can be mapped to a multi-set of $pd$ numbers\footnote{E.g.\ We can identify the polynomial $\pmb Y_{11}\pmb Y_{12}\pmb Y_{21}\pmb Y_{11}$ with the multi-set of its indices $\{1,1,1,1,1,1,2,2\}$.}, where multiple $\alpha$'s could map to the same multi-set of $pd$ numbers\footnote{E.g.\ $\pmb Y_{11}\pmb Y_{12}\pmb Y_{21}\pmb Y_{11}$, $\pmb Y_{11}\pmb Y_{12}\pmb Y_{12}\pmb Y_{11}$, and $\pmb Y_{11}\pmb Y_{11}\pmb Y_{11}\pmb Y_{22}$ all map to $\{1,1,1,1,1,1,2,2\}$.}.
Thus, one can first pick a multi-set and then go over the different $\alpha$'s corresponding to all possible permutations\footnote{E.g.\ $(1,1,1,2,2,1,1,1) \equiv \pmb Y_{11}\pmb Y_{12}\pmb Y_{21}\pmb Y_{11}$ and $(1,1,1,2,1,2,1,1) \equiv \pmb Y_{11}\pmb Y_{12}\pmb Y_{12}\pmb Y_{11}$ are counted differently.}.
Under constraint of $\mathbbm{1}_{even(c(\alpha))}$, a multi-set is \emph{valid} (contributes a non-zero term to the summation) only when the multiplicity of each number is even.
So, one can view the summation as a process of first choosing $s$ distinct coordinates from $[n]$ such that each coordinate is used a non-zero even number of times when forming a multi-set of $pd$ numbers.
Naturally, we have $1 \leq s \leq \left\lfloor pd/2 \right\rfloor \leq n$ and $s(\alpha) = s$.
For a fixed choice of $s$ coordinates, $\sum_{\substack{\beta_1 + \ldots + \beta_s = pd/2\\ \beta_1 \neq 0, \ldots, \beta_s \neq 0}} \binom{pd}{2 \beta_1, \ldots, 2 \beta_s}$ sums over all valid multi-sets involving $s$ entries of $[n]$.
However, since every permutation of a fixed multi-set corresponds to a possibly repeated $\alpha$'s, we divide by $\binom{d}{\alpha_1, \ldots, \alpha_{n^p}}$\footnote{E.g.\ Suppose $\beta_1 = 3$, $\beta_2 = 1$ and $pd=8$ in the combinatorial summation. $\binom{pd}{2 \beta_1, \ldots, 2 \beta_s}$ will include permutations such as $(1,1,1,2,1,2,1,1)$ and $(1,1,1,1,1,2,1,2)$. However, both of $(1,1,1,2,1,2,1,1)$ and $(1,1,1,1,1,2,1,2)$ actually refer to the same $\alpha$ term since $\pmb Y_{11}\pmb Y_{12}\pmb Y_{12}\pmb Y_{11} = \pmb Y_{11}\pmb Y_{11}\pmb Y_{12}\pmb Y_{12}$.}.
Finally, each such $\alpha$ is then scaled by $\left( \Pi_{i=1}^{n^p} \frac{1}{(\alpha_i)!} \right)$.
\end{proof}

\subsubsection*{Example illustrating \cref{claim:counting}}

We illustrate the counting process with an example where $p=2$, $n=2$, and $d=3$.
Denote $\alpha, \beta, \gamma \in [n]^2$ as three distinct coordinates of $\pmb Y$.
By picking entries $\{\pmb Y_\alpha, \pmb Y_\beta, \pmb Y_\gamma\}$, the corresponding Hermite polynomial $h_1(\pmb Y_\alpha) h_1(\pmb Y_\beta) h_1(\pmb Y_\gamma) = \pmb Y_\alpha \pmb Y_\beta \pmb Y_\gamma$ is multi-linear.
With repeated entries such as $\{\pmb Y_\alpha, \pmb Y_\alpha, \pmb Y_\beta\}$ and $\{\pmb Y_\alpha, \pmb Y_\alpha, \pmb Y_\alpha\}$, the corresponding Hermite polynomials are $h_2(\pmb Y_\alpha) h_1(\pmb Y_\beta)$ and $h_3(\pmb Y_\alpha)$ respectively.

By the constraint of $\mathbbm{1}_{even(c(\alpha))}$, it suffices to only consider choices such that there are an even number of 1's and 2's.
Ignoring permutations, there are 10 such selections.
Including permutations, there are $\binom{3}{3} \cdot 2 + \binom{3}{2,1} \cdot 6 + \binom{3}{1,1,1} \cdot 2 = 32$ such selections.
Note that only the last 2 are multi-linear.
\begin{description}
    \item [1 distinct:]
        $\{\pmb Y_{11}, \pmb Y_{11}, \pmb Y_{11}\}$,
        $\{\pmb Y_{22}, \pmb Y_{22}, \pmb Y_{22}\}$ 
    \item [2 distinct:]
        $\{\pmb Y_{11}, \pmb Y_{11}, \pmb Y_{22}\}$,
        $\{\pmb Y_{11}, \pmb Y_{12}, \pmb Y_{12}\}$,
        $\{\pmb Y_{11}, \pmb Y_{21}, \pmb Y_{21}\}$,
        $\{\pmb Y_{11}, \pmb Y_{22}, \pmb Y_{22}\}$,
        $\{\pmb Y_{12}, \pmb Y_{12}, \pmb Y_{22}\}$,
        $\{\pmb Y_{21}, \pmb Y_{21}, \pmb Y_{22}\}$
    \item [3 distinct:]
        $\{\pmb Y_{11}, \pmb Y_{12}, \pmb Y_{21}\}$,
        $\{\pmb Y_{12}, \pmb Y_{21}, \pmb Y_{22}\}$
\end{description}

We first compute the summation on the left hand side of \cref{claim:counting}.
An $\alpha$ with 1 distinct entry such as $\{\pmb Y_{11}, \pmb Y_{11}, \pmb Y_{11}\}$ contributes $\left( \frac{k}{n} \right)^2 \frac{1}{3!}$ to the summation.
With 2 distinct entries, such as $\{\pmb Y_{11}, \pmb Y_{12}, \pmb Y_{12}\}$, we get $\left( \frac{k}{n} \right)^4 \frac{1}{1!2!}$.
Finally, each multi-linear polynomial contributes $\left( \frac{k}{n} \right)^4 \frac{1}{1!1!1!}$.
So,
\[
\sum_{\Abs{\alpha} = d} \mathbbm{1}_{even(c(\alpha))} \left( \frac{k}{n} \right)^{2 s(\alpha)} \left( \Pi_{i=1}^{n^p} \frac{1}{(\alpha_i)!} \right)
= \frac{2}{3!} \left( \frac{k}{n} \right)^2 + \frac{6}{1!2!} \left( \frac{k}{n} \right)^4 + \frac{2}{1!1!1!} \left( \frac{k}{n} \right)^4
= \frac{1}{3} \left( \frac{k}{n} \right)^2 + 5 \left( \frac{k}{n} \right)^4
\]

On the right hand side of \cref{claim:counting}, we count by viewing the selection of 3 entries of $Y$ as filling up $pd=6$ slots with values from $\{1,2\}$:
\begin{align*}
& \frac{1}{d!} \sum_{s=1}^{\left\lfloor pd/2 \right\rfloor} \binom{n}{s} \left( \frac{k}{n} \right)^{2s} \sum_{\substack{\beta_1 + \ldots + \beta_s = pd/2\\ \beta_1 \neq 0, \ldots, \beta_s \neq 0}} \binom{pd}{2 \beta_1, \ldots, 2 \beta_s}\\
= & \frac{1}{3!} \sum_{s=1}^2 \binom{2}{s} \left( \frac{k}{n} \right)^{2s} \sum_{\substack{\beta_1 + \ldots + \beta_s = 3\\ \beta_1 \neq 0, \ldots, \beta_s \neq 0}} \binom{6}{2 \beta_1, \ldots, 2 \beta_s}\\
= & \frac{1}{6} \binom{2}{1} \left( \frac{k}{n} \right)^{2} \binom{6}{6} + \frac{1}{6} \binom{2}{2} \left( \frac{k}{n} \right)^{4} \left[ \binom{6}{2,4} + \binom{6}{4,2} \right]\\
= &\frac{1}{3} \left( \frac{k}{n} \right)^{2} + 5 \left( \frac{k}{n} \right)^{4}
\end{align*}

\begin{claim}
\label{claim:xax}
For $x > 0$ and $0 < a < 1$, we have $xa^x \leq \min \left\{ x, \frac{1}{e \ln (1/a)} \right\}$.
\end{claim}
\begin{proof}
When $0 < a < 1$, we have $xa^x \leq x$ trivially.
For $x > 0$ and $0 < a < 1$, we see that $(1/a)^x > 0$.
So,
\[
\left( \frac{1}{a} \right)^x \geq e \ln \left( \frac{1}{a} \right)^x = ex \ln \left( \frac{1}{a} \right)
\iff
xa^x \leq \frac{1}{e \ln \left( \frac{1}{a} \right)}
\]
Thus, $xa^x \leq \min \left\{ x, \frac{1}{e \ln (1/a)} \right\}$.
\end{proof}

\hermiteexpectation*
\begin{proof}[Proof of \cref{lem:hermite-expectation}]
For fixed multi-index $\alpha = (\alpha_1, \ldots, \alpha_{n^p})$ such that $\Abs{\alpha} = d$, we now compute $\E_{H_1} h_{\alpha}(\pmb Y)$.
\begin{align*}
& \E_{H_1} h_{\alpha}(\pmb Y)\\
= & \E_{H_1} \Pi_{i=1}^{n^p} h_{\alpha_i} (\pmb Y_{\phi(i)}) && \text{Product of Hermite polys}\\
= & \E_{\pmb x} \E_{\pmb W_{\phi(i)} \sim N(0,1)} \Pi_{i=1}^{n^p} h_{\alpha_i} (\pmb Y_{\phi(i)}) && \text{Definition of $H_1$}\\
= & \E_{\pmb x} \Pi_{i=1}^{n^p} \E_{\pmb W_{\phi(i)} \sim N(0,1)} h_{\alpha_i} (\pmb Y_{\phi(i)}) && \text{Independence of $\pmb W$ entries}\\
= & \E_{\pmb x} \Pi_{i=1}^{n^p} \E_{\pmb W_{\phi(i)} \sim N(0,1)} h_{\alpha_i} (\pmb W_{\phi(i)} + \lambda \Pi_{j \in \phi(i)} x_j) && \text{Definition of $\pmb Y_{\alpha_i}$}\\
= & \E_{\pmb x} \Pi_{i=1}^{n^p} \E_{\pmb z \sim N(\lambda \Pi_{j \in \phi(i)} x_j,1)} h_{\alpha_i} (z) && \text{Translation property of Hermite}\\
= & \E_{\pmb x} \Pi_{i=1}^{n^p} \sqrt{\frac{1}{(\alpha_i)!}} ( \lambda \Pi_{j \in \phi(i)} x_j)^{\alpha_i} && \text{Expectation of deg $\alpha_i$ Hermite on $\pmb z \sim (\mu,1)$}\\
= & \left( \Pi_{i=1}^{n^p} \sqrt{\frac{1}{(\alpha_i)!}} \right) \lambda^{d} \E_{\pmb x} \Pi_{i=1}^{n^p} \Pi_{j \in \phi(i)} x_j^{\alpha_i} && \text{$\sum_i^{n^p} \alpha_i = \Abs{\alpha} = d$}\\
= & \left( \Pi_{i=1}^{n^p} \sqrt{\frac{1}{(\alpha_i)!}} \right) \lambda^{d} \E_{\pmb x} \Pi_{j=1}^{n} x_j^{c_j} && \text{Definition of $c(\alpha) = (c_1, \ldots, c_n)$}\\
= & \left( \Pi_{i=1}^{n^p} \sqrt{\frac{1}{(\alpha_i)!}} \right) \lambda^{d} \mathbbm{1}_{even(c(\alpha))} \left( \frac{k}{n} \right)^{s(\alpha)} k^{-\frac{pd}{2}}
\end{align*}

The last equality is because $\E_{\pmb x} \Pi_{j=1}^{n} x_j^{c_j}  = 0$ if there is an odd $c_j$.
So, for $\Abs{\alpha} = d$,
\[
\left( \E_{H_1} h_{\alpha}(\pmb Y) \right)^2
= \lambda^{2d} k^{-pd} \mathbbm{1}_{even(c(\alpha))} \left( \frac{k}{n} \right)^{2 s(\alpha)} \left( \Pi_{i=1}^{n^p} \frac{1}{(\alpha_i)!} \right)
\]
\end{proof}

\hermiteexpectationupperbound*
\begin{proof}[Proof of \cref{lem:hermite-expectation-upperbound}]
To upper bound $\sum_{\Abs{\alpha} \leq D} \left( \E_{H_1} [f_{\alpha}(\pmb Y)] \right)^2$, we use an equality that relates the counting of $\pmb Y$ entries with coordinates of the signal $\pmb x$.
For a fixed $d$, it can be shown (see \cref{claim:counting}) that
\[
\sum_{\Abs{\alpha} = d} \mathbbm{1}_{even(c(\alpha))} \left( \frac{k}{n} \right)^{2 s(\alpha)} \left( \Pi_{i=1}^{n^p} \frac{1}{(\alpha_i)!} \right)
= \frac{1}{d!} \sum_{s=1}^{\left\lfloor pd/2 \right\rfloor} \binom{n}{s} \left( \frac{k}{n} \right)^{2s} \sum_{\substack{\beta_1 + \ldots + \beta_s = pd/2\\ \beta_1 \neq 0, \ldots, \beta_s \neq 0}} \binom{pd}{2 \beta_1, \ldots, 2 \beta_s}
\]
This allows us to perform combinatoric arguments on the coordinates of the signal $\pmb x$ instead of over the tensor coordinates of $\pmb Y$.

\begin{align*}
& \sum_{\Abs{\alpha} \leq D} \left( \E_{H_1} [f_{\alpha}(\pmb Y)] \right)^2\\
= & \sum_{d=1}^D \sum_{\Abs{\alpha} = d} \lambda^{2d} k^{-pd} \mathbbm{1}_{even(c(\alpha))} \left( \frac{k}{n} \right)^{2 s(\alpha)} \left( \Pi_{i=1}^{n^p} \frac{1}{(\alpha_i)!} \right) && \text{From above}\\
= & \sum_{d=1}^D \lambda^{2d} k^{-pd} \frac{1}{d!} \sum_{s=1}^{\left\lfloor pd/2 \right\rfloor} \binom{n}{s} \left( \frac{k}{n} \right)^{2s} \sum_{\substack{\beta_1 + \ldots + \beta_s = pd/2\\ \beta_1 \neq 0, \ldots, \beta_s \neq 0}} \binom{pd}{2 \beta_1, \ldots, 2 \beta_s} && \text{\cref{claim:counting}}\\
\leq & \sum_{d=1}^D \frac{\lambda^{2d} k^{-pd}}{d!} \sum_{s=1}^{pd/2} \binom{n}{s} \left( \frac{k}{n} \right)^{2s} \sum_{\substack{\beta_1 + \ldots + \beta_s = pd/2\\ \beta_1 \neq 0, \ldots, \beta_s \neq 0}} \binom{pd}{2 \beta_1, \ldots, 2 \beta_s} && \text{Drop floor}\\
\leq & \sum_{d=1}^D \frac{\lambda^{2d} k^{-pd}}{d!} \sum_{s=1}^{pd/2} \binom{n}{s} \left( \frac{k}{n} \right)^{2s} \sum_{\beta_1 + \ldots + \beta_s = pd/2} \binom{pd}{2 \beta_1, \ldots, 2 \beta_s} && \text{Drop $\beta_i \neq 0$}\\
\leq & \sum_{d=1}^D \frac{\lambda^{2d} k^{-pd}}{d!} \sum_{s=1}^{pd/2} \binom{n}{s} \left( \frac{k}{n} \right)^{2s} \sum_{\gamma_1 + \ldots + \gamma_s = pd} \binom{pd}{\gamma_1, \ldots, \gamma_s} && \text{Drop ``evenness constraint''}\\
= & \sum_{d=1}^D \frac{\lambda^{2d} k^{-pd}}{d!} \sum_{s=1}^{pd/2} \binom{n}{s} \left( \frac{k}{n} \right)^{2s} s^{pd} && \text{Multinomial theorem}\\
\leq & \sum_{d=1}^D \frac{\lambda^{2d} k^{-pd}}{d!} \sum_{s=1}^{pd/2} \left( \frac{ek^2}{sn} \right)^s s^{pd} && \text{$\binom{n}{s} \leq \left( \frac{en}{s} \right)^s$}\\
= & \sum_{d=1}^D \frac{\lambda^{2d}}{d!} \sum_{s=1}^{pd/2} \left( \frac{ek^2}{sn} \right)^s \left( \frac{s}{k} \right)^{pd} 
\end{align*}
\end{proof}

\sumupperbound*
\begin{proof}[Proof of \cref{lem:sum-upperbound}]
We will first push all terms into $[ \cdots ]^{pd}$ and then upper bound the terms inside\footnote{This works because the terms inside are greater than 0 and $pd \geq 1$.}.
We start by recalling three useful inequalities:
\begin{itemize}
	\item For $x > 0$, we have $x^{\frac{1}{x}} \leq 2$.
	\item For $x > 0$ and $0 < a < 1$, we have $xa^x \leq \min \left\{ x, \frac{1}{e \ln (1/a)} \right\}$.
	\item For $x \geq \frac{1}{e}$, we have $\min \left\{ \frac{1}{2}, \frac{1}{e \ln x} \right\} \leq \frac{1}{1 + \Abs{\ln(x)}}$.
\end{itemize}
Using the first inequality, we get
\[
\left( \frac{ek^2}{sn} \right)^s \left( \frac{s}{k} \right)^{pd}
= \left[ \left( \frac{ek^2}{sn} \right)^{\frac{s}{pd}} \frac{s}{k} \right]^{pd}
= \left[ \left( \frac{ek^2}{npd} \right)^{\frac{s}{pd}} \left( \frac{pd}{s} \right)^{\frac{s}{pd}} \frac{s}{k} \right]^{pd}
\leq \left[ 2 \left( \frac{ek^2}{npd} \right)^{\frac{s}{pd}} \frac{s}{k} \right]^{pd}
\]
When $ek^2 \geq npd$, we use $s \leq pd/2$ to get
\[
\left( \frac{ek^2}{npd} \right)^{\frac{s}{pd}} \frac{s}{k}
\leq \left( \frac{ek^2}{npd} \right)^{\frac{pd/2}{pd}} \frac{pd/2}{k}
= \sqrt{\frac{epd}{4n}}
\leq \sqrt{\frac{pd}{n}}
\]
When $ek^2 < npd$, we use the second and third inequalities\footnote{Observe that $0 < \frac{s}{pd} \leq \frac{1}{2}$, $0 < \frac{ek^2}{npd} < 1$, and $\frac{1}{e} \leq 1 < \frac{npd}{ek^2}$.} to get
\[
\left( \frac{ek^2}{npd} \right)^{\frac{s}{pd}} \frac{s}{k}
= \frac{s}{pd} \left( \frac{ek^2}{npd} \right)^{\frac{s}{pd}} \frac{pd}{k}
\leq \min \left\{ \frac{1}{2}, \frac{1}{e \ln \left( \frac{npd}{ek^2} \right)} \right\} \cdot \frac{pd}{k}
\leq \frac{pd}{k \left( 1 + \Abs{\ln \left( \frac{npd}{ek^2} \right)} \right)}
\]
Putting together, we see that
\[
\left( \frac{ek^2}{sn} \right)^s \left( \frac{s}{k} \right)^{pd}
\leq \left[ 2 \max \left\{ \sqrt{\frac{pd}{n}}, \frac{pd}{k \left( 1 + \Abs{\ln \left( \frac{npd}{ek^2} \right)} \right) } \right\} \right]^{pd}
= \left[ \frac{2pd}{\min \left\{ \sqrt{npd},\; k \left( 1 + \Abs{\ln \left( \frac{npd}{ek^2} \right)} \right) \right\}} \right]^{pd}
\]
\end{proof}

\subsection{Proofs for information-theoretic lower bound}
\label{sec:info-lower-bound-proofs}

\rademacherepsnetlowerbound*
\begin{proof}[Proof of \cref{lem:rademacher-eps-net-lower-bound}]
For $x, x' \in U_k$, let us denote
$\alpha = \Abs{ \left\{ i \in [n] : i \in \cI_{x} \cap \cI_{x'} \text{ and } x_i = x'_i \right\} }$ be the intersecting indices with agreeing signs,
$\beta = \Abs{ \left\{ i \in [n] : i \in \cI_{x} \cap \cI_{x'} \text{ and } x_i = -x'_i \right\} }$ be the intersecting indices with disagreeing signs, and 
$\gamma = \Abs{ \left\{ i \in [n] : i \not\in \cI_{x} \cap \cI_{x'} \right\} }$ be the non-intersecting indices.
By definition,
$\alpha \geq 0$, $\beta \geq 0$,
$\gamma \geq 0$, $\alpha + \beta = \Abs{ \cI_{x} \cap \cI_{x'} }$,
$\alpha + \beta + \gamma = 2k - \Abs{ \cI_{x} \cap \cI_{x'} } = \Abs{\cI_{x}} + \Abs{ \cI_{x'} } - \Abs{ \cI_{x} \cap \cI_{x'} }$, and
$\gamma = 2(k - \Abs{ \cI_{x} \cap \cI_{x'} })$.
Then, for $x, x' \in U_k$,
\[
\Norm{x - x'}_2
= \sqrt{\beta \left( \frac{2}{\sqrt{k}} \right)^2 + \gamma \left( \frac{1}{\sqrt{k}} \right)^2}
= \sqrt{\frac{4 \beta + \gamma}{k}}
\geq \sqrt{\frac{\gamma}{k}}
= \sqrt{2 - \frac{2 \Abs{ \cI_{x} \cap \cI_{x'} }}{k}}
\]
So, $\Norm{x-x'}_2 \leq \eps$ implies that $\Abs{ \cI_{x} \cap \cI_{x'} } \geq k (1-\frac{\eps^2}{2})$.
This means that for any \emph{fixed} $x \in U_k$, there are at most\footnote{First pick $i$ out of $k$ coordinates of $x$ to be different, then pick the $i$ different coordinates amongst the $n-k$ coordinates outside of $\cI_{x}$. The summation is from 0 to $\left\lfloor \eps^2 k / 2 \right\rfloor$ because we need to have $\Abs{ \cI_{x} \cap \cI_{x'} } \geq k (1-\frac{\eps^2}{2})$.} $\sum_{i=0}^{\left\lfloor \eps^2 k / 2 \right\rfloor} \binom{k}{i} \binom{n-k}{i}$ vectors in $U_k$ (including $x$ itself) that are of distance at most $\eps$ from $x$.
By definition of covering number, we know that
\[
N(U_k, \Norm{\cdot}_2, \eps) \cdot \sum_{i=0}^{\left\lfloor \eps^2 k / 2 \right\rfloor} \binom{k}{i} \binom{n-k}{i}
\geq \Abs{U_k}
= 2^k \binom{n}{k}
\]
Thus, to argue that $N(U_k, \Norm{\cdot}_2, \eps) \geq \left( \frac{n-k}{k} \right)^{k \left( 1-\frac{\eps^2}{2} \right)}$, it suffices to show
\[
2^k \binom{n}{k}
\geq \left( \frac{n-k}{k} \right)^{k \left( 1-\frac{\eps^2}{2} \right)} \cdot \sum_{i=0}^{\left\lfloor \eps^2 k / 2 \right\rfloor} \binom{k}{i} \binom{n-k}{i}
\]
Observe that since $n \geq 2k$, the term $\binom{n-k}{i}$ increases with $i$:
\begin{align*}
    \sum_{i=0}^{\left\lfloor \eps^2 k / 2 \right\rfloor} \binom{k}{i} \binom{n-k}{i}
    & \leq \binom{n-k}{\left\lfloor \eps^2 k / 2 \right\rfloor} \sum_{i=0}^{\left\lfloor \eps^2 k / 2 \right\rfloor} \binom{k}{i} && \text{$(\star)$}\\
    & \leq \binom{n-k}{\left\lfloor \eps^2 k / 2 \right\rfloor} \sum_{i=0}^k \binom{k}{i} && \text{$\left\lfloor \eps^2 k / 2 \right\rfloor \leq k$}\\
    & = \binom{n-k}{\left\lfloor \eps^2 k / 2 \right\rfloor} 2^k && \text{Binomial theorem}
\end{align*}
where $(\star)$ is because $n \geq 2k$ and $\eps \in (0,1]$ implies that $n - k \geq \eps^2 k$ so $\binom{n-k}{i} \leq \binom{n-k}{\left\lfloor \eps^2 k / 2 \right\rfloor}$ for $0 \leq i \leq \left\lfloor \eps^2 k / 2 \right\rfloor$.
Thus, it suffices to show
\[
\binom{n}{k}
\geq \left( \frac{n-k}{k} \right)^{k \left( 1-\frac{\eps^2}{2} \right)} \cdot \binom{n-k}{\left\lfloor \eps^2 k / 2 \right\rfloor}
\]
We will now show that $\frac{\binom{n}{k}}{\binom{n-k}{\left\lfloor \eps^2 k / 2 \right\rfloor}} \geq \left( \frac{n-k}{k} \right)^{k \left( 1-\frac{\eps^2}{2} \right)}$:
\begin{align*}
    \frac{\binom{n}{k}}{\binom{n-k}{\left\lfloor \eps^2 k / 2 \right\rfloor}}
    & = \frac{n!}{k!(n-k)!} \frac{(\left\lfloor \eps^2 k / 2 \right\rfloor)! (n-k-\left\lfloor \eps^2 k / 2 \right\rfloor)!}{(n-k)!}\\
    & = \frac{n!}{(n-k)!} \frac{(\left\lfloor \eps^2 k / 2 \right\rfloor)!}{k!} \frac{(n-k-\left\lfloor \eps^2 k / 2 \right\rfloor)!}{(n-k)!}\\
    & = \left[ (n) \cdot \ldots \cdot (n-k+1) \right] \cdot \left[ \frac{1}{(k) \cdot \ldots \cdot (k- \left\lfloor \eps^2 k / 2 \right\rfloor + 1)} \right] \cdot\\
    & \quad \left[ \frac{1}{(n-k) \cdot \ldots \cdot (n - k - \left\lfloor \eps^2 k / 2 \right\rfloor + 1)} \right]\\
    & \geq (n-k)^k \left( \frac{1}{k} \right)^{\left\lfloor \eps^2 k / 2 \right\rfloor} \left( \frac{1}{n-k} \right)^{\left\lfloor \eps^2 k / 2 \right\rfloor}\\
    & \geq (n-k)^k \left( \frac{1}{k} \right)^{\frac{\eps^2 k}{2}} \left( \frac{1}{n-k} \right)^{\frac{\eps^2 k}{2}}\\
    & \geq (n-k)^k \left( \frac{1}{k} \right)^{k \left( 1 - \frac{\eps^2}{2} \right)} \left( \frac{1}{n-k} \right)^{\frac{\eps^2 k}{2}}\\
    & = \left( \frac{n-k}{k} \right)^{k \left( 1-\frac{\eps^2}{2} \right)}
\end{align*}
where the last inequality is because $\eps \leq 1$ implies that $1 - \frac{\eps^2}{2} \geq \frac{\eps^2}{2}$.
\end{proof}

\KLupperbound*
\begin{proof}[Proof of \cref{lem:KL-upperbound}]
Define $vec(T)$ as vectorization of a tensor from $\otimes^p\R^n$ to $\R^{n^p}$.
Then, for $u \in \cS^{n-1}_k$, we see that $\Norm{ vec(\lambda u^{\otimes p}) }_2^2 = \lambda^2$ and the distribution $\bbP_{\pmb Y \sim \cY \mid u}$ follows the distribution of a Gaussian vector $\pmb g \sim N( vec(\lambda u^{\otimes p}), I_{n^p} )$.

For two Gaussian vectors $\pmb g \sim N(\mu_0, I_{n^p})$ and $\pmb h \sim N(\mu_1, I_{n^p})$, we know that $D_{KL}(\pmb g \Vert \pmb h) = \frac{1}{2} (\mu_1 - \mu_0)^\top (\mu_1 - \mu_0) = \frac{1}{2} \Norm{ \mu_1 - \mu_0 }_2^2 \leq \frac{1}{2} ( \Norm{ \mu_1 }_2 + \Norm{ \mu_0 }_2 )^2$ by triangle inequality\footnote{We get an equality if $\mu_1 = -\mu_0$.}.

Thus, $D_{KL} \left( \bbP_{\pmb Y \sim \cY \mid u} \Bigm\Vert \bbP_{\pmb Y \sim \cY \mid v} \right) \leq \frac{1}{2} ( \lambda^2 + \lambda^2 )^2 = 2\lambda^2$.
\end{proof}

\end{document}